\documentclass[runningheads]{llncs}
\pdfoutput=1

\usepackage{eccv}
\usepackage{eccvabbrv}
\usepackage{graphicx}
\usepackage{booktabs}
\usepackage{mathtools}
\usepackage{pifont}
\usepackage{multirow}
\usepackage{colortbl}
\usepackage{tabularx}
\newcolumntype{C}{>{\centering\arraybackslash}X}
\usepackage{float}
\usepackage{placeins}
\usepackage{tcolorbox}
\tcbuselibrary{skins}
\newcommand{\cN}{\mathcal{N}}
\newcommand{\cS}{\mathcal{S}}
\newcommand{\cU}{\mathcal{U}}
\newcommand{\tA}{\tilde{A}}
\newcommand{\nidx}{\nu}
\newcommand{\sidx}{\sigma}
\newcommand{\vphi}{\varphi}

\newtcolorbox{desideratum}[2][]{%
  colback=eccvblue!3,
  colframe=eccvblue!30,
  boxrule=0.4pt,
  arc=0pt,
  left=6pt, right=6pt, top=4pt, bottom=4pt,
  fonttitle=\bfseries,
  coltitle=black,
  colbacktitle=eccvblue!7,
  toptitle=2pt, bottomtitle=2pt,
  title={Desideratum #2},
  #1
}

\newtcolorbox{theoremframe}[1][]{%
  enhanced,
  colback=white,
  colframe=white,
  boxrule=0pt,
  borderline west={2.5pt}{0pt}{eccvblue!50},
  arc=0pt,
  left=10pt, right=4pt, top=4pt, bottom=4pt,
  #1
}

\newtcolorbox{remarkbox}[2][]{%
  enhanced,
  colback=white,
  colframe=white,
  boxrule=0pt,
  borderline west={2pt}{0pt}{black!18},
  arc=0pt,
  left=10pt, right=4pt, top=4pt, bottom=4pt,
  fonttitle=\bfseries\itshape,
  coltitle=black!60,
  title={#2},
  #1
}

\newtcolorbox{proofstep}[2][]{%
  enhanced,
  colback=white,
  colframe=white,
  boxrule=0pt,
  borderline west={1.5pt}{0pt}{black!25},
  arc=0pt,
  left=8pt, right=2pt, top=2pt, bottom=2pt,
  before skip=4pt, after skip=4pt,
  before upper={\noindent\textbf{#2}\ },
  #1
}

\newtcolorbox{operatorbox}[2][]{%
  colback=black!2,
  colframe=black!20,
  boxrule=0.3pt,
  arc=0pt,
  left=6pt, right=6pt, top=4pt, bottom=4pt,
  fonttitle=\bfseries,
  coltitle=black!70,
  colbacktitle=black!4,
  toptitle=2pt, bottomtitle=2pt,
  title={#2},
  #1
}

\newtcolorbox{structframe}[2][]{%
  colback=white,
  colframe=black,
  boxrule=0.4pt,
  arc=0pt,
  left=8pt, right=8pt, top=6pt, bottom=6pt,
  before upper={\noindent\textbf{\textit{#2}}\enspace},
  #1
}

\newtcolorbox{consequenceframe}[2][]{%
  enhanced,
  colback=white,
  colframe=black,
  boxrule=0.4pt,
  arc=0pt,
  left=8pt, right=8pt, top=6pt, bottom=6pt,
  attach boxed title to top left={yshift=-\tcboxedtitleheight/2, xshift=8pt},
  boxed title style={%
    colback=white, colframe=white, boxrule=0pt,
    left=3pt, right=3pt, top=0pt, bottom=0pt},
  fonttitle=\bfseries\itshape,
  coltitle=black,
  title={#2},
  #1
}

\newtcolorbox{synthesisframe}[2][]{%
  enhanced,
  colback=white,
  colframe=white,
  boxrule=0pt,
  borderline north={0.4pt}{0pt}{black},
  borderline south={0.4pt}{0pt}{black},
  arc=0pt,
  left=4pt, right=4pt, top=6pt, bottom=6pt,
  before upper={\noindent\ifx\relax#2\relax\else\textbf{\textit{#2}}\enspace\fi},
  #1
}

\usepackage[accsupp]{axessibility}
\graphicspath{{figures/}}
\usepackage[pagebackref,breaklinks,colorlinks,citecolor=eccvblue]{hyperref}
\usepackage{orcidlink}
\usepackage[capitalize,noabbrev]{cleveref}

\begin{document}

\title{Attribution Upsampling should Redistribute, Not Interpolate}

\titlerunning{Attribution Upsampling should Redistribute, Not Interpolate}

\author{Vincenzo Buono\inst{1} \and
Peyman Sheikholharam Mashhadi\inst{1} \and
Mahmoud Rahat\inst{1} \and
Prayag Tiwari\inst{1} \and
Stefan Byttner\inst{1}}

\authorrunning{V.~Buono et al.}

\institute{Halmstad University, Halmstad, Sweden\\
\email{\{vincenzo.buono, peyman.mashhadi, mahmoud.rahat, prayag.tiwari, stefan.byttner\}@hh.se}}

\maketitle

\begin{abstract}
  Attribution methods in explainable AI rely on upsampling techniques that were designed for natural images, \emph{not} saliency maps. Standard bilinear and bicubic interpolation systematically corrupts attribution signals through aliasing, ringing, and boundary bleeding, producing spurious high-importance regions that misrepresent model reasoning. We identify that the core issue is treating attribution upsampling as an \emph{interpolation} problem that operates in isolation from the model's reasoning, rather than a \emph{mass redistribution} problem where model-derived semantic boundaries must govern how importance flows. We present \textbf{Universal Semantic-Aware Upsampling (USU)}, a principled method that reformulates upsampling through ratio-form mass redistribution operators, provably preserving attribution mass and relative importance ordering. Extending the axiomatic tradition of feature attribution to upsampling, we formalize four desiderata for faithful upsampling and prove that interpolation structurally violates three of them. These same three force any redistribution operator into a ratio form; the fourth selects the unique potential within this family, yielding USU. Controlled experiments on models with known attribution priors verify USU's formal guarantees; evaluation across ImageNet, CIFAR-10, and CUB-200 confirms consistent faithfulness improvements and qualitatively superior, semantically coherent explanations.
	\keywords{Explainable AI \and Feature Attribution \and Saliency Map Upsampling}
\end{abstract}

\begin{figure}[t]
  \centering
  \includegraphics[width=\linewidth]{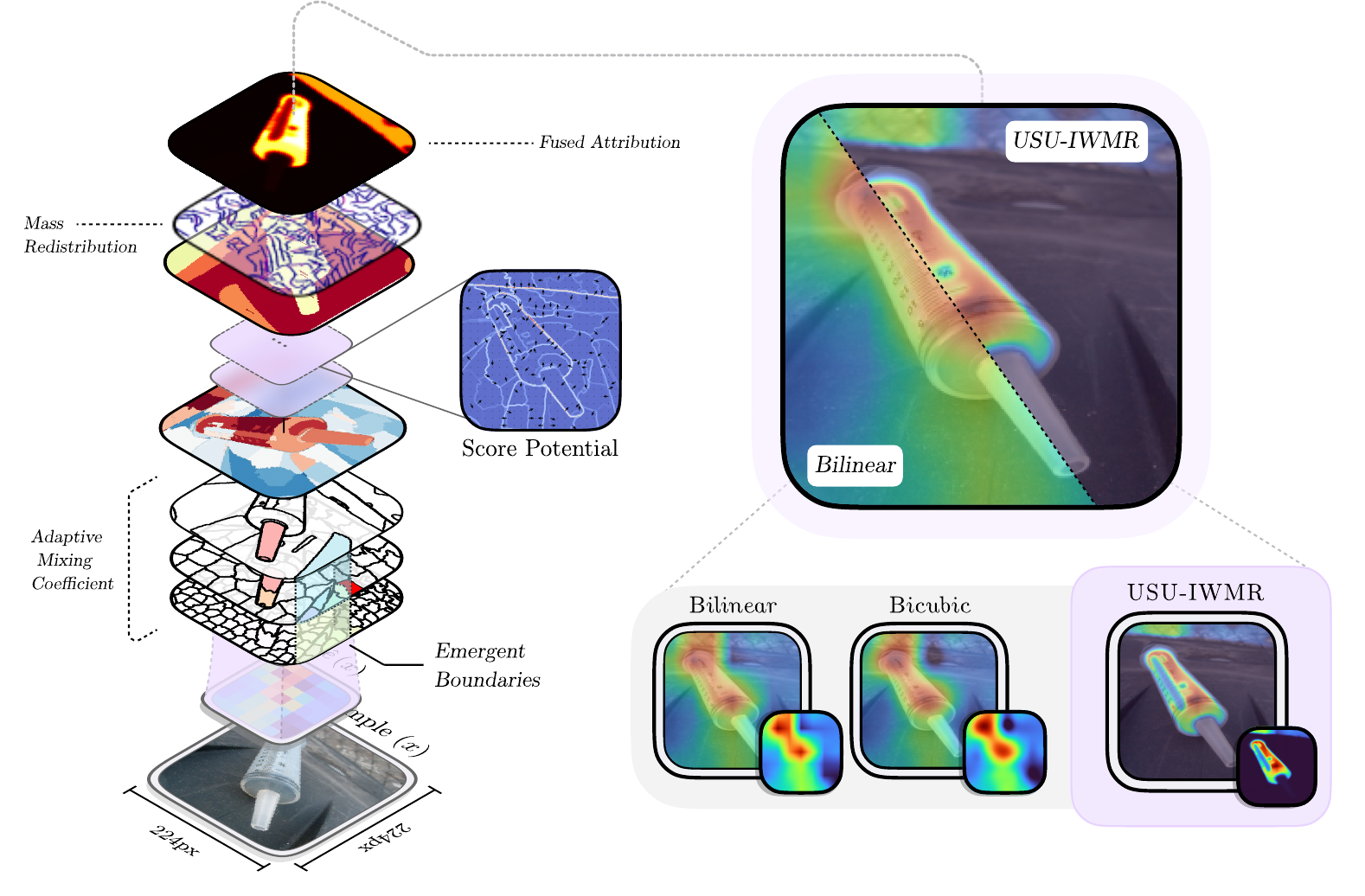}
  \vspace{-4pt}
  \caption{\textbf{USU replaces interpolation with semantically guided mass redistribution.} \emph{Left:} The pipeline decomposes upsampling into score-potential computation, emergent boundary detection, and ratio-form mass redistribution. \emph{Right:} Standard interpolation (bilinear, bicubic) bleeds attribution across object boundaries; USU-IWMR produces sharp, semantically coherent saliency maps that faithfully reflect model reasoning.}
  \label{fig:hero}
\end{figure}

\begin{synthesisframe}{}
Interpolation reconstructs a continuous signal. The resulting artifacts distort what practitioners \emph{see} and corrupt what faithfulness metrics \emph{measure}. USU \emph{redistributes mass} according to prediction relevance: the unique operator satisfying conservation, monotonicity, score-invariance, and locality.
\end{synthesisframe}

\section{Introduction}
\label{sec:intro}

\begin{figure}[t]
  \centering
  \includegraphics[width=\linewidth]{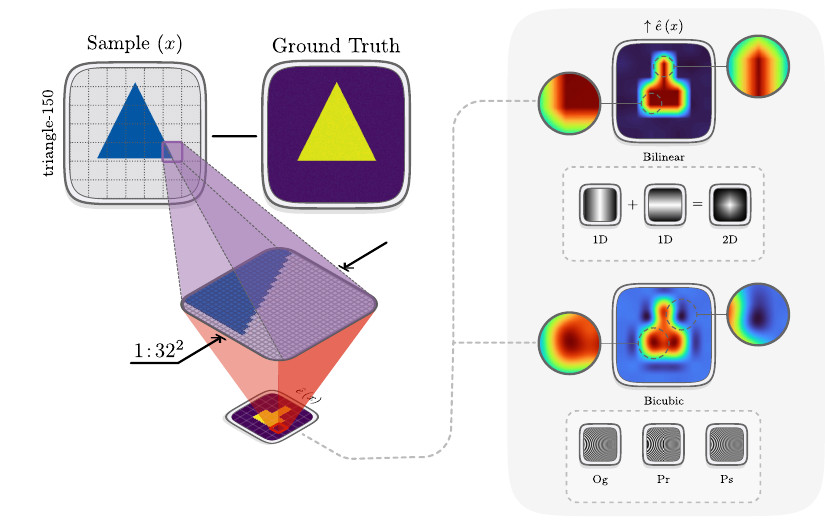}
  \caption{\textbf{Upsampling corrupts saliency maps through reconstruction errors.} Classical saliency upsampling kernels can introduce reconstruction errors, leading to incorrect visual explanations. RRL-constrained~\cite{7db2afdc5eb5db46cc64185d0a51ed079b0976e8} model on synthetic shapes; see the supplemental material for setup and extended results.}
  \label{fig:upsampling-errors}
\end{figure}

Neural networks are increasingly deployed in safety-critical vision applications, from medical imaging to autonomous driving. Understanding their decisions has become a prerequisite for deployment. Feature attribution methods such as Grad-CAM \cite{5582bebed97947a41e3ddd9bd1f284b73f1648c2}, Integrated Gradients \cite{f302e136c41db5de1d624412f68c9174cf7ae8be}, SHAP \cite{442e10a3c6640ded9408622005e3c2a8906ce4c2}, Layer-wise Relevance Propagation \cite{a002e71561c90767240672f357b7d9e6d4d95186}, and RISE \cite{d00c7fc5201405d5411b5ad3da93c5575ce8f10e} have become the standard tools for this purpose, producing importance maps that highlight which input regions drive a prediction. These methods universally produce coarse-resolution maps that must be upsampled to input resolution before interpretation. This ubiquitous upsampling step has been treated as a solved problem. We show it is not: standard upsampling techniques \emph{systematically corrupt} the explanations they produce.

Standard interpolation methods such as bilinear and bicubic upsampling were designed to reconstruct natural images \cite{de5ca64429d18710232ecc0da54998e1a401fcbe,6731642027d82aa37289331a77f06c0871689c39}. They were never intended for explainability. Attribution maps demand fundamentally different properties: \emph{importance mass} must be conserved so the explanation accounts for the full prediction, redistribution must follow the model's own \emph{semantic structure} rather than spatial coordinates, and upsampled attributions must remain \emph{spatially local}, ensuring the fine-grained explanation remains faithful to the original coarse attribution. Standard interpolation violates all three.

This mismatch manifests as concrete artifacts that corrupt explanation fidelity (\Cref{fig:upsampling-errors}). Aliasing causes semantic boundaries to be incorrectly smoothed; ringing artifacts from bicubic interpolation's negative weights create spurious high-attribution regions; boundary bleeding allows attribution to leak across semantic boundaries. These artifacts \emph{distort} what practitioners see and \emph{corrupt} what faithfulness metrics measure. The fundamental issue is that interpolation operates in isolation from the model's reasoning, treating all spatial locations as interchangeable regardless of their importance to the prediction. The resulting explanation is faithful to neither the model's reasoning nor the coarse attribution it was meant to refine. As \Cref{fig:identical-saliencies} shows, three samples with entirely different shapes and ground-truth importance regions yield identical upsampled explanations: interpolation cannot distinguish what the model attends to from what it ignores.

We formalize four desiderata that any \emph{faithful} upsampling operator must satisfy: neighbourhood completeness (D1), segment monotonicity (D2), constant conditioning strength (D3), and locality (D4). These extend the axiomatic foundations of feature attribution theory \cite{f302e136c41db5de1d624412f68c9174cf7ae8be,4f5e178d2cecc2c10fcb82ddd9c20a16d210bdbc,8102cad29ec8e808c7395ac6ee668da495f07206} to the upsampling domain.

We prove that widely adopted interpolation methods such as bilinear and bicubic upsampling violate (D1), (D2), and (D4), corrupting attribution signals and breaking faithfulness. Their kernel support crosses neighbourhood boundaries, leaking mass; their weights depend on geometry rather than semantics; and boundary pixels depend on adjacent neighbourhoods. We derive the \textbf{Universal Semantic-Aware Upsampling (USU)} operator as the unique ratio-form solution satisfying all four desiderata, with uniqueness following from Luce's Choice Axiom \cite{64cce5f2fb89e06be57b32294df1078dd8227bad}: the Independence of Irrelevant Alternatives forces any mass-conserving, monotone, local redistribution operator into the ratio form. A temperature parameter controls redistribution sharpness. We also introduce \textbf{Soft IWMR}, a variant that intentionally relaxes (D1) to allow mass flow between neighbourhoods based on global importance, useful when coarse attributions contain systematic biases.

Whereas prior work has treated attribution upsampling as a solved implementation detail, we provide the first formal treatment, from axiomatic characterization through structural incompatibility results to a provably unique solution. Specifically:
\begin{enumerate}
    \item \textbf{Axiomatic Framework.} We formalize four desiderata (D1)--(D4) for faithful attribution upsampling, extending the axiomatic tradition of feature attribution theory. We prove that bilinear and bicubic interpolation structurally violate three of these desiderata (D1), (D2), (D4), establishing that interpolation-based upsampling is fundamentally unsuitable for attribution maps.

    \item \textbf{USU Operator.} We derive Universal Semantic-Aware Upsampling as the \emph{unique} ratio-form operator satisfying all four desiderata, with uniqueness following from Luce's Independence of Irrelevant Alternatives. A temperature parameter controls redistribution sharpness, and the Soft IWMR variant relaxes strict mass conservation when coarse attributions contain systematic biases.

    \item \textbf{Empirical Validation.} We first verify USU's formal guarantees on constrained models where ground-truth attributions are known, achieving 2--4 orders of magnitude improvement in Infidelity. We then confirm these guarantees transfer to real-world settings: evaluations across ImageNet, CIFAR-10, and CUB-200 with three architectures (VGG16, ResNet50, ViT-B-16) demonstrate consistent faithfulness improvements across diverse quantitative metrics and qualitatively superior, semantically coherent explanations.
\end{enumerate}

\section{Related Work}
\label{sec:related}

\subsection{Spatial Resolution in Feature Attribution}
\label{sec:related-attribution}

The attribution ecosystem spans diverse methodological families with markedly different native resolutions \cite{93f566bda2b8a419f83d0b532e7d048fc5ce4ba5}. The most widely deployed family, CAM and its variants \cite{31f9eb39d840821979e5df9f34a6e92dd9c879f2,5582bebed97947a41e3ddd9bd1f284b73f1648c2,2c1b79a13087a8e9bc2a4446384145e6f85d4820,c2ecbcd66b46b03e44cbc51bc0be7f7d3a09f4b9,2e257b5f0a0148076e9161202880b022e521ac59}, operates on final convolutional feature maps (typically $7{\times}7$ or $14{\times}14$) and upsamples via bilinear interpolation. Vision Transformer interpretability methods \cite{76a9f336481b39515d6cea2920696f11fb686451,0acd7ff5817d29839b40197f7a4b600b7fba24e4,7ec5f207263100ea2d45db595712f611a74bafd9} face an analogous constraint, producing patch-level attributions that must be tiled or interpolated to pixel resolution. Black-box approaches such as RISE \cite{d00c7fc5201405d5411b5ad3da93c5575ce8f10e} use coarse random masks, while LIME \cite{c0883f5930a232a9c1ad601c978caede29155979} and SHAP \cite{442e10a3c6640ded9408622005e3c2a8906ce4c2} typically operate on superpixel features. Even methods that nominally produce full-resolution maps encounter the gap in practice: Integrated Gradients \cite{f302e136c41db5de1d624412f68c9174cf7ae8be}, DeepLIFT \cite{1a2118bed729579528deb51e745d58dd3629baf6}, and Layer-wise Relevance Propagation \cite{a002e71561c90767240672f357b7d9e6d4d95186} are noisy and often applied to intermediate layers, while FullGrad \cite{0e09d47723cccf0a9bad7cb1badaa367ea37fe14} aggregates multi-layer maps that are individually upsampled.
Whether structurally coarse or nominally pixel-level, these methods rely on interpolation designed for image reconstruction when mapping attributions to input resolution. The faithfulness of this upsampling step itself remains unexamined: no prior work formalizes what properties attribution upsampling must satisfy or whether interpolation can guarantee them.

\subsection{Saliency Map Refinement}
\label{sec:related-refinement}

The community has recognized the resolution problem and proposed various remedies, though all remain heuristic. \emph{Multi-layer aggregation} methods (LayerCAM \cite{9f82a887ebfb7e69705feff355012f26549d6a0b}, CAMERAS \cite{9fdfc2df42a159dbf1134ad9b367c2afd904ecd0}, Poly-CAM \cite{a3cc917b7e0607c1dbc8ac511a7e2655188b7855}) combine maps from multiple network depths but still upsample each layer via interpolation. \emph{Augmentation-based} approaches (Augmented Grad-CAM \cite{3e79a70159cccec297a6b13872702518a4d8a116}, Jitter-CAM \cite{e4071bdfd6dd2c330fbe4e3a0f934cb0d7aa7a4d}, EVET \cite{20d1d32dd959ec5b935c2131cbee3b8d38df4d6f}, Gradual Extrapolation \cite{32d1e2ae372a194e0c91d96cf3c6b9cdf34d0625}) break the coarse grid via geometric or network-layer transformations but provide no formal mass guarantees. \emph{Segment-based} methods are closest to USU: XRAI \cite{faf3a22627f198ce56671b6f8a9b3d5cc3164a91} aggregates Integrated Gradients over image segments and SATs \cite{58eb3c25d372df30c4a125d7fc71efa8073773cc} leverage SAM for semi-global explanations, but neither formally redistributes mass. Taimeskhanov~\etal~\cite{3d1b5bce71c0bc02e35e0c3b0f82a6e3fb6d6c63} demonstrate that CAMs can ``see through walls,'' showing the problem extends deeper than resolution alone. None of these methods isolate upsampling as a principled operation; USU fills this gap as a standalone, axiomatically derived operator.

\subsection{Axiomatic Foundations in Explainability}
\label{sec:related-axiomatics}

The axiomatic approach has been remarkably successful in uniquely determining attribution methods from natural properties. Shapley values \cite{8fd17bf36bc22477bb2237c2be6e3212b753969d,442e10a3c6640ded9408622005e3c2a8906ce4c2} are uniquely characterized by efficiency, symmetry, linearity, and the null player axiom, the gold standard of axiomatic explainability. Integrated Gradients \cite{f302e136c41db5de1d624412f68c9174cf7ae8be} admit multiple independent axiomatic characterizations via completeness, sensitivity, implementation invariance, and related properties \cite{b0ceba96db8b7b71beb2f0906568ffc143fddbf9,c5f34ad0e0dfb818f3304f318512271b77f1b954}. Unified decomposition-and-allocation frameworks \cite{a3272b24b42da1cbcdc60c02bcaef8ceca1cf2b1,a84c7a715e294c4536651adf9f848f8b30ea1a2d,508aa9f057845d5b22c8d7e1ce2034b4131c9fdd,0e6384b81172e3e3c69518b793d131fc5af356b8,a05041a2ce7b69bbb85f9c9944d736be9fd44257,4e2e782126fb4015c25897ad631e7ba2b893f54f} reveal Shapley and IG as special cases of a common paradigm, and interaction extensions \cite{3da37f966df6b3117be7070e2fc90e9bdeccb35e,798ea191aad9401462b405fde1a6cefb4fe53fd5,ec73937dddc1d5573072d39cc5131178c796dc03} lift these axioms to higher-order effects. Yet impossibility results \cite{8102cad29ec8e808c7395ac6ee668da495f07206}, conflicting operationalizations \cite{4f5e178d2cecc2c10fcb82ddd9c20a16d210bdbc}, and axiom-level critiques \cite{364f02eff4f10ea602d86fd8c98c8694b76f46fd,ae0d489f14ec350131f11d9a0f874303f312bea5,a9a8464ea537d13b1deca7532f57d97383b7a154} highlight that axioms constrain but context matters, motivating domain-specific axiomatizations.

A key theoretical tool from outside explainability is Luce's Choice Axiom \cite{64cce5f2fb89e06be57b32294df1078dd8227bad}, specifically the Independence of Irrelevant Alternatives (IIA), which uniquely characterizes the exponential/softmax form in discrete choice theory \cite{251646ea79c2a3973311852975f70a64878fffec}. Yet this powerful axiomatic program has never been applied to attribution upsampling. We extend it, showing that natural desiderata for faithful attribution refinement uniquely determine a ratio-form operator.

\subsection{Faithfulness Evaluation}
\label{sec:related-evaluation}

Numerous perturbation-based metrics \cite{d00c7fc5201405d5411b5ad3da93c5575ce8f10e,6df11b0bb0244d4d36e8955436067cc5d19734fa,5f614777d25efd14b7426e99cb2544f2d6be133e,04e4183fd18093ab34e8c82bc4403b75901e7cec} and standardized toolkits \cite{30e776268268e84becd2863b0632247da61238b9,8c3babcb113081d0c4cfdfbd6fb3518a595892c9,868e35374cb9c0fc6e4cfb17f96835aefcf520cc} evaluate attribution quality, yet no consensus has emerged: metrics frequently disagree \cite{1f69f1b02aee5d0d0ad94889c54b4de9d34d2405,e46992da8b56028df35a2d350bf285bc03e7a24d}, are sensitive to design choices \cite{eacdc609f79804d5067616f1b538a127c58c069a}, and can fail even with known ground truth \cite{c8b557903a87b0feaca37003b104a3ba22dda4cc}. A complementary strategy constructs models with known attribution priors \cite{7db2afdc5eb5db46cc64185d0a51ed079b0976e8,96d1b48417113b41a67557c50e3e5405c3c86a55,e0714f730d557d68e552284b3c9bd6567f116ca5,7d7a17da63a107424ddeec91612341100d8de42e} so that ground-truth explanations are available by design. Crucially, existing metrics evaluate the full attribution pipeline end-to-end, conflating attribution-method quality with upsampling fidelity. We adopt both strategies: controlled experiments isolate the upsampling contribution, while standard benchmarks confirm that better upsampling translates to better explanations.

\begin{figure}[t]
  \centering
  \includegraphics[width=\linewidth]{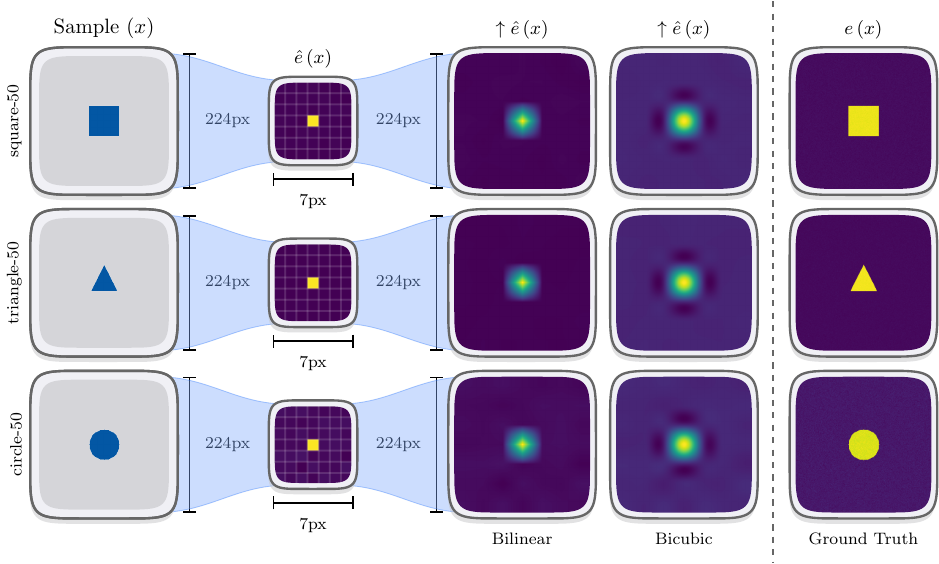}
  \caption{\textbf{Identical saliency for distinct inputs.} Each row shows (left) the input, (middle) the raw heatmap, (right) the upsampled saliency map, and (rightmost) the ground-truth attribution priors. Despite distinct ground truths, the upsampled explanations converge to a single pattern: interpolation cannot distinguish what the model attends to from what it ignores. RRL-constrained~\cite{7db2afdc5eb5db46cc64185d0a51ed079b0976e8} model on synthetic shapes; see the supplemental material for setup and extended results.}
  \label{fig:identical-saliencies}
\end{figure}

\section{From Interpolation to Redistribution}
\label{sec:method}

Standard upsampling treats a coarse attribution map as a signal to be smoothly reconstructed. But \textbf{attribution maps are not images}. They carry \emph{mass}: each neighbourhood's total attribution accounts for the model's reasoning about that region. Any upsampling that \emph{creates or destroys} that mass \emph{distorts} the explanation it was meant to refine. The question is not \emph{``what value belongs at this pixel?''}\ but \emph{``how should this neighbourhood's mass be allocated among its pixels?''} Formalizing this yields a complete theory: interpolation provably fails three necessary requirements (\Cref{sec:interpolation-failure}), four desiderata characterize faithful redistribution (\Cref{sec:desiderata}), and a uniqueness argument forces the solution (\Cref{sec:uniqueness}). \Cref{fig:identical-saliencies} illustrates this failure: three distinct inputs yield identical upsampled explanations.

\subsection{Problem Setup}
\label{sec:setup}

\paragraph{Lattice and partitions.}
Let $\Omega = \{0,\ldots,H{-}1\}\times\{0,\ldots,W{-}1\}$ denote the \textbf{pixel lattice}. An \textbf{upsampling neighbourhood system} $\cN = \{N_0,\ldots,N_{K-1}\}$ partitions $\Omega$ into $K$ disjoint sets, each the receptive field of one coarse attribution value; we write $\nidx\colon\Omega\to\{0,\ldots,K{-}1\}$ for the index function ($\nidx(x)=k$ iff $x\in N_k$). A \textbf{segmentation partition} $\cS = \{S_0,\ldots,S_{P-1}\}$ independently partitions $\Omega$ into $P$ semantic segments with index function $\sidx\colon\Omega\to\{0,\ldots,P{-}1\}$. The partition $\cN$ encodes the \emph{spatial} structure of the coarse attribution; $\cS$ provides the model's \emph{semantic} structure that guides how attribution mass is allocated among pixels.

\paragraph{Attribution and mass.}
A \textbf{coarse attribution} $A\colon\Omega\to\mathbb{R}$ is piecewise constant on neighbourhoods, $A(x)=a_{\nidx(x)}$. The \textbf{neighbourhood mass} aggregates attribution within each neighbourhood:
\begin{equation}
  M_k = \sum_{x \in N_k} A(x).
  \label{eq:neighbourhood-mass}
\end{equation}
\textbf{Segment scores} $s\colon\{0,\ldots,P{-}1\}\to[0,1]$ assign normalized importance to each segment. Given the tuple $(A,s,\cN,\cS)$, we seek an upsampling operator producing a \textbf{fused attribution} $\tA\colon\Omega\to\mathbb{R}$:
\begin{equation}
  \tA = \cU(A,s,\cN,\cS)\colon\Omega\to\mathbb{R}.
  \label{eq:upsampling-operator}
\end{equation}

\paragraph{Error model.}
Upsampling quality is captured by two complementary error types. The \textbf{$\alpha$-error} (spurious attribution) and \textbf{$\beta$-error} (signal loss) at pixel $x$ relative to ground-truth $A^*$ are:
\begin{equation}
  \alpha(x)=\max\bigl(0,|\tA(x)|-|A^*(x)|\bigr),\quad
  \beta(x)=\max\bigl(0,|A^*(x)|-|\tA(x)|\bigr).
  \label{eq:errors}
\end{equation}
At each pixel exactly one can be nonzero.

\subsection{Why Interpolation Fails}
\label{sec:interpolation-failure}

Before formalizing what faithful redistribution requires, we establish what it cannot be: \textbf{interpolation}. Classical kernels~\cite{de5ca64429d18710232ecc0da54998e1a401fcbe,6731642027d82aa37289331a77f06c0871689c39} optimize for smooth signal reconstruction, ignoring both the \emph{attribution mass} that must be conserved and the \emph{semantic structure} that encodes prediction relevance. \textbf{Three structural violations} make this incompatibility precise.

\begin{theorem}[Interpolation Violates Mass Conservation]
\label{thm:interp-d1}
Any interpolation kernel whose support extends across neighbourhood boundaries with non-zero weight violates mass conservation. Specifically, there exists a coarse attribution $A$ such that $\sum_{x \in N_k} \tA(x) \neq M_k$.
\end{theorem}

Mass leaks across neighbourhood boundaries whenever a kernel's support crosses them. Full proof in the supplemental material.

\begin{theorem}[Interpolation Ignores Semantics]
\label{thm:interp-d2}
Interpolation is score-independent: for any two segment score functions $s_1, s_2$, the interpolated attributions are identical. Consequently, interpolation violates segment monotonicity.
\end{theorem}

Interpolation weights are determined entirely by spatial coordinates; the operator is independent of the score function $s$. The model's semantic structure is invisible to the operator: within any neighbourhood where the kernel assigns unequal weights to two pixels, one can always choose scores so that the higher-scored pixel receives less attribution, directly violating (D2).

\begin{theorem}[Interpolation Violates Locality]
\label{thm:interp-d4}
When kernel support extends outside a neighbourhood, boundary pixel attributions depend on external data. Specifically, there exist $A_1, A_2$ agreeing on $N_k$ with $\tA_1(x_0) \neq \tA_2(x_0)$ for boundary pixel $x_0 \in N_k$.
\end{theorem}

Adjacent neighbourhoods influence boundary attributions whenever a kernel spans the boundary. Full proof in the supplemental material.

\begin{synthesisframe}{}
These failures reflect a fundamental mismatch: interpolation reconstructs a continuous signal; attribution upsampling must allocate mass according to prediction relevance.
\end{synthesisframe}

\subsection{Desiderata for Faithful Upsampling}
\label{sec:desiderata}

Having established what faithful redistribution \emph{cannot} be, we now formalize what it \emph{demands}. In the axiomatic tradition of feature attribution~\cite{f302e136c41db5de1d624412f68c9174cf7ae8be,442e10a3c6640ded9408622005e3c2a8906ce4c2}, we identify four desiderata that any faithful upsampling operator must satisfy.

\begin{desideratum}{(D1) Neighbourhood Completeness}
\label{def:d1}
For all $k$:
\begin{equation}
\sum_{x \in N_k} \tA(x) = M_k
\label{eq:d1}
\end{equation}
\end{desideratum}

Attribution mass is conserved within each neighbourhood. This directly extends the \emph{completeness} axiom that characterizes both Shapley values~\cite{442e10a3c6640ded9408622005e3c2a8906ce4c2} and Integrated Gradients~\cite{f302e136c41db5de1d624412f68c9174cf7ae8be}: just as Shapley values ensure player contributions sum to the coalition value, (D1) ensures pixel attributions sum to the neighbourhood mass.

\begin{desideratum}{(D2) Segment Monotonicity}
\label{def:d2}
For $x, y \in N_k$ with $M_k \geq 0$:
\begin{equation}
s(\sidx(x)) \geq s(\sidx(y)) \implies \tA(x) \geq \tA(y)
\label{eq:d2}
\end{equation}
\end{desideratum}

Higher segment scores yield higher attributions, reflecting the \emph{sensitivity} principle from Integrated Gradients: if the model considers a segment more important, its pixels should receive proportionally more attribution.

\begin{desideratum}{(D3) Constant Conditioning Strength}
\label{def:d3}
For potential $\vphi\colon [0,1] \to \mathbb{R}_{>0}$ and temperature $\varepsilon > 0$:
\begin{equation}
\frac{\vphi(s + \Delta)}{\vphi(s)} = \exp(\Delta / \varepsilon), \quad \text{independent of } s
\label{eq:d3}
\end{equation}
\end{desideratum}

The relative effect of score changes is uniform across the score range, ensuring that a fixed score difference $\Delta$ has the same redistributive effect whether it occurs among low-score or high-score segments. Alternative potentials lack this: a power-law $\vphi(s)=s^\gamma$ compresses ratios at high scores and amplifies them near zero (see the supplemental material). (D3) demands score-invariant sensitivity, a uniformity property related to Luce's constant-ratio rule~\cite{64cce5f2fb89e06be57b32294df1078dd8227bad}: the multiplicative effect of a score difference $\Delta$ on weights must be the same at every base score. Without this, identical score gaps produce different redistributive effects depending on the score range (see the supplemental material for concrete examples).

\begin{desideratum}{(D4) Locality}
\label{def:d4}
\begin{equation}
\tA(x) = f\bigl(M_{\nidx(x)},\, \{s(\sidx(y))\}_{y \in N_{\nidx(x)}}\bigr)
\label{eq:d4}
\end{equation}
for some function $f$. That is, $\tA(x)$ depends only on the mass and segment scores within its own neighbourhood.
\end{desideratum}

Changes outside the neighbourhood do not affect attributions inside it, ensuring computational tractability and preventing error propagation from distant regions.

\subsection{Uniqueness and the USU Operator}
\label{sec:uniqueness}

We now show that \emph{any} operator satisfying (D1), (D2), and (D4) must take a specific ratio form, paralleling how completeness, symmetry, and linearity uniquely determine Shapley values. Redistribution requires that scaling a neighbourhood's mass scales every pixel's share identically; otherwise the allocation rule itself depends on the total, conflating \emph{how much} to distribute with \emph{where}. This \textbf{linearity in mass} gives $\tA(x) = M_k \cdot w_k(x)$ for weight function $w_k$ depending only on segment scores, paralleling the linearity axiom for Shapley values~\cite{442e10a3c6640ded9408622005e3c2a8906ce4c2}.

\begin{theoremframe}
\begin{theorem}[Uniqueness of the Ratio Form]
\label{thm:uniqueness}
Any linear-in-mass operator satisfying (D1) conservation, (D4) locality, and (D2) monotonicity must have the ratio form:
\begin{equation}
\cU(M_k, s)(x) = M_k \cdot \frac{\vphi(s(\sidx(x)))}{\sum_{y \in N_k} \vphi(s(\sidx(y)))}
\label{eq:ratio-form}
\end{equation}
for some strictly positive, strictly monotone potential $\vphi\colon [0,1] \to \mathbb{R}_{>0}$.
\end{theorem}

\begin{proof}[Proof sketch]
The proof proceeds in three steps.

\begin{proofstep}{\mbox{Step~1} (IIA from Locality).}
(D4) implies Independence of Irrelevant Alternatives: the ratio $R(a,b) = w(s_a)/w(s_b)$ is independent of other alternatives in the neighbourhood.
\end{proofstep}

\begin{proofstep}{\mbox{Step~2} (Cauchy Functional Equation).}
IIA yields transitivity $R(a,b)\cdot R(b,c)=R(a,c)$, the multiplicative Cauchy equation. Under continuity and (D2) monotonicity, the unique solution is $R(a,b)=\vphi(a)/\vphi(b)$ for continuous, strictly monotone $\vphi$.
\end{proofstep}

\begin{proofstep}{\mbox{Step~3} (Conservation Forces Normalization).}
Conservation~(D1) requires $\sum_{x\in N_k}w_k(x)=1$ with $w_k(x)\propto\vphi(s(\sidx(x)))$, which forces the ratio form. Full proof in the supplemental material.
\end{proofstep}
\end{proof}
\end{theoremframe}

This parallels how Luce's IIA uniquely characterizes the multinomial logit model in choice theory~\cite{64cce5f2fb89e06be57b32294df1078dd8227bad,ea84a6ef34223f4f0d8b64555a6b6cec312b8fce,f10afb5d088114389a71148d39aab0c57820b02e,251646ea79c2a3973311852975f70a64878fffec}. The characterization is \emph{layered}: (D1), (D2), and (D4) force the ratio form for \emph{any} strictly monotone potential $\vphi$; (D3) then uniquely pins down the potential, so that all four together admit exactly one operator (up to temperature). Alternative potentials sacrifice (D3)'s score-invariance but retain (D1)--(D2)--(D4); see the supplemental material for analysis of this tradeoff.

\paragraph{The USU Operator.}
USU instantiates the ratio form with the \emph{tensor potential}, which satisfies (D3) by direct computation: $\vphi(s+\Delta)/\vphi(s)=\exp(\Delta/\varepsilon)$, independent of~$s$.

\begin{operatorbox}{USU Operator}
Given temperature $\varepsilon > 0$, the \textbf{tensor potential}
\begin{equation}
\vphi(s) = \exp\!\left(\frac{s - 0.5}{\varepsilon}\right),
\label{eq:tensor-potential}
\end{equation}
defines the \textbf{normalized weight} for pixel $x$ in neighbourhood $k$:
\begin{equation}
w_k(x) = \frac{\vphi(s(\sidx(x)))}{\sum_{y \in N_k} \vphi(s(\sidx(y)))},
\label{eq:normalized-weight}
\end{equation}
and the \textbf{USU fused attribution}:
\begin{equation}
\tA(x) = M_{\nidx(x)} \cdot w_{\nidx(x)}(x).
\label{eq:usu-operator}
\end{equation}
\end{operatorbox}

USU operates as a semantic softmax: segment scores pass through the non-linear potential $\vphi$, are normalized within each neighbourhood, and redistribute mass proportionally. Though linear in mass (scaling the input scales the output), the operator is non-linear in scores through~$\vphi$.

\begin{theorem}[(D1)--(D4) Satisfaction]
\label{thm:usu-all}
USU with the tensor potential satisfies all four desiderata. Specifically: (D1)~weight normalization ensures $\sum_{x\in N_k}\tA(x)=M_k$; (D2)~the exponential is strictly increasing, so higher scores yield higher weights when $M_k\geq 0$; (D3)~the tensor potential satisfies constant conditioning by construction; (D4)~$\tA(x)$ depends only on $M_{\nidx(x)}$ and scores within $N_{\nidx(x)}$.
\end{theorem}

\begin{corollary}[Global Conservation]
\label{cor:global}
Since $\cN$ partitions $\Omega$, summing (D1) over all neighbourhoods yields global mass conservation:
\begin{equation}
\sum_{x \in \Omega} \tA(x) = \sum_{x \in \Omega} A(x).
\label{eq:global-conservation}
\end{equation}
\end{corollary}

\paragraph{Error Minimization.}
(D1) ensures zero net mass error per neighbourhood when the coarse attribution faithfully aggregates ground truth: if $M_k = \sum_{x\in N_k}A^*(x)$, then both $\alpha$- and $\beta$-errors from \eqref{eq:errors} satisfy $\Delta^-_k = \Delta^+_k = 0$. USU provides the tightest possible mass-level guarantee. Full error analysis in the supplemental material.

\subsection{The Attribution Basin Problem}
\label{sec:soft-iwmr}

Strict mass conservation (D1) guarantees that each neighbourhood's attribution total is preserved exactly, faithful to the coarse explanation. But consider a neighbourhood whose attribution mass lies in a \emph{basin}: a region of moderate importance enclosed by low-attribution neighbours. (D1) traps this mass in place even when the model's segment scores indicate that semantically important pixels nearby deserve more. The result is a \textbf{pit in the attribution landscape} that no amount of within-neighbourhood redistribution can fill. This exposes a tension between two notions of faithfulness: fidelity to the coarse heatmap versus fidelity to the model's own semantic judgements.

This exposes a tension between \emph{heatmap faithfulness} (D1: each neighbourhood's mass matches the coarse attribution exactly) and \emph{model faithfulness} (the model's segment scores govern mass allocation across the full lattice). When the coarse attribution is itself imperfect, strict (D1) preserves these imperfections.

\paragraph{Importance-weighted redistribution.}
Soft IWMR (Importance-Weighted Mass Redistribution) relaxes (D1) by redistributing the global mass $M_{\mathrm{total}}=\sum_k M_k$ according to neighbourhood importance. Each neighbourhood receives importance $\lambda_k = \max_{p\in P_k}s_p$, where $P_k=\{p:N_k\cap S_p\neq\emptyset\}$ indexes segments intersecting~$N_k$.

\begin{operatorbox}{Soft IWMR Operator}
The \textbf{importance potential} and \textbf{redistribution weights}:
\begin{align}
\Lambda(\lambda) &= \exp\!\left(\frac{\lambda-0.5}{\varepsilon_\Lambda}\right), &
\rho_k &= \frac{\Lambda(\lambda_k)\cdot|N_k|}{\sum_j\Lambda(\lambda_j)\cdot|N_j|}, \label{eq:redistribution-weights}
\end{align}
yield \textbf{redistributed masses} and the \textbf{IWMR fused attribution}:
\begin{align}
\tilde{M}_k &= M_{\mathrm{total}}\cdot\rho_k, &
\tA^{\mathrm{IWMR}}(x) &= \tilde{M}_{\nidx(x)}\cdot w_{\nidx(x)}(x). \label{eq:iwmr-operator}
\end{align}
The within-neighbourhood weights $w_k(x)$ are identical to USU \eqref{eq:normalized-weight}; only the mass input changes from $M_k$ to $\tilde{M}_k$.
\end{operatorbox}

Soft IWMR intentionally violates (D1) but satisfies a \emph{generalized} completeness (D1-G): $\sum_{x \in N_k} \tA^{\mathrm{IWMR}}(x) = \tilde{M}_k$. It preserves (D2), (D3), and a semi-local form of (D4), as well as global conservation with zero-sum mass flow ($\sum_k (\tilde{M}_k - M_k) = 0$). This modularity is by design: each desideratum addresses an independent concern (mass conservation, semantic ordering, calibration, locality), so relaxing one for a specific operational reason does not compromise the others. Just as the layered characterization permits alternative potentials $\vphi$ while retaining (D1)--(D2)--(D4) guarantees (\Cref{thm:uniqueness}), Soft IWMR relaxes (D1) to correct coarse-attribution bias while preserving the remaining three. Full formal statement and proofs in the supplemental material.

\paragraph{Temperature control.}
$\varepsilon_\Lambda$ governs the strength of inter-neighbourhood redistribution: small values concentrate mass in high-importance neighbourhoods, while large values approach uniformity. When all importances are equal, $\tilde{M}_k = M_k$ and IWMR reduces to standard USU.

\paragraph{Hierarchical boundary refinement.}
The segmentation partition $\cS$ that defines these scores has so far been assumed, not derived. For the resulting explanations to be faithful, boundaries must emerge from the model's own reasoning, refined where the model's scores transition across segments and left coarse where scores are homogeneous. No single segmentation scale achieves this: a coarse partition groups semantically distinct regions under a shared score; a fine partition fragments coherent ones. The key insight is to let \emph{score heterogeneity itself} select boundary resolution: start from the coarsest partition, identify where importance transitions occur, and refine only there. The result is a multi-scale attribution whose boundaries shift to match the model's own reasoning. Since each recursion step applies USU, neighbourhood completeness propagates inductively and the full pipeline preserves total attribution mass. Full formalization (boundary operator, depth recursion, and stopping criterion) appears in the supplemental material.

\FloatBarrier
\section{Experiments}
\label{sec:experiments}

We evaluate USU along two axes: (1)~controlled synthetic tasks with known attribution priors, verifying formal guarantees (\Cref{sec:synthetic-validation}); and (2)~standard benchmarks with pretrained models, confirming practical faithfulness gains (\Cref{sec:faithfulness}; \Cref{fig:results-hero}).

\begin{figure}[t!]
  \centering
  \includegraphics[width=\linewidth]{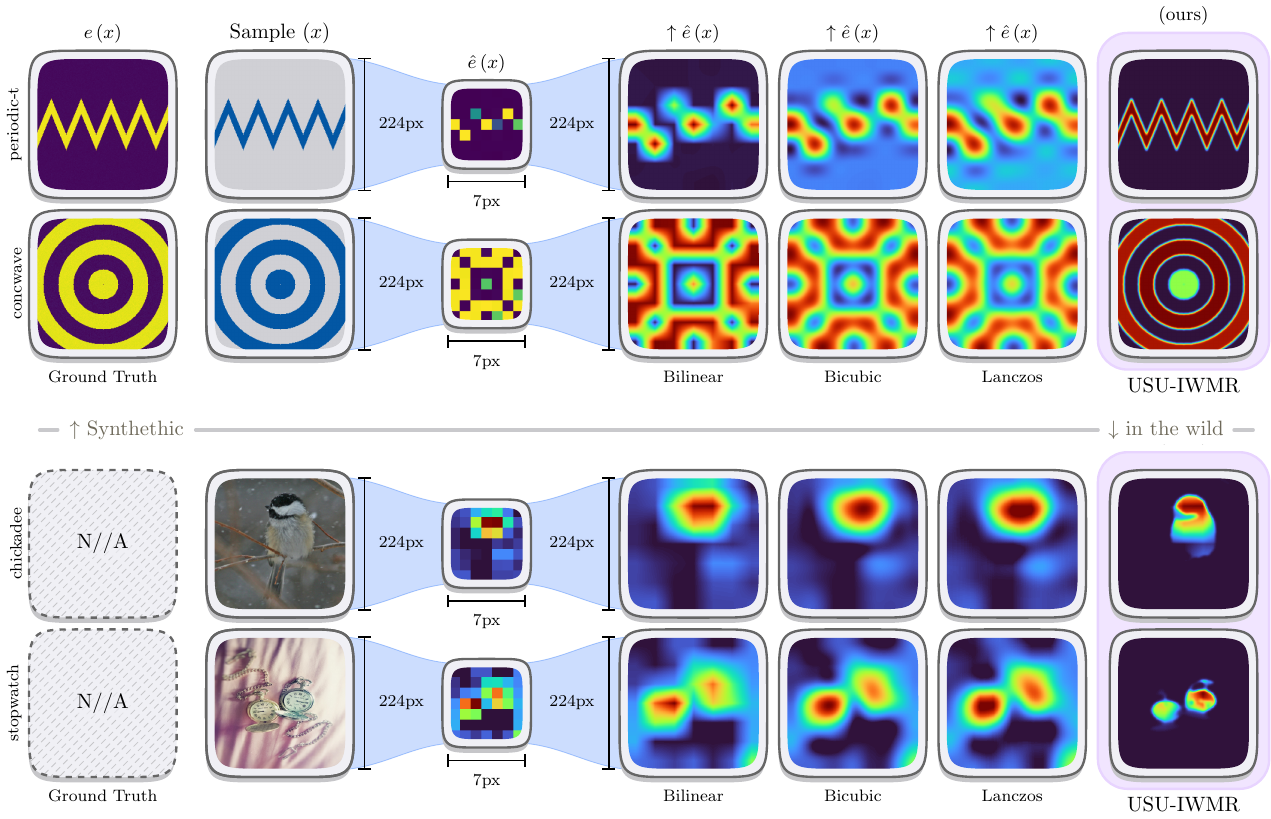}
  \caption{\textbf{Qualitative comparison.} Top row: synthetic patterns with known ground-truth attribution priors. Bottom row: ImageNet examples with GradCAM attributions. Interpolation methods (bilinear, bicubic, Lanczos) produce ringing artifacts, boundary bleeding, and aliasing; USU preserves semantic boundaries and concentrates attribution within model-relevant regions.}
  \label{fig:results-hero}
\end{figure}

\subsection{Controlled Validation with Attribution Priors}
\label{sec:synthetic-validation}

\paragraph{Setup.}
We design synthetic shape classification tasks (circle/triangle/square) with known ground-truth attributions, training vanilla and RRL-constrained~\cite{7db2afdc5eb5db46cc64185d0a51ed079b0976e8} CNN/MLP architectures. Coarse attributions at resolutions $\{4, 7, 14\}$ are upsampled via interpolation baselines, USU, USU-IWMR, and oracle variants. Full setup and metrics in the supplemental material.
\vspace{-1.5ex}
\paragraph{Results.}
\Cref{tab:synthetic-results} confirms USU's formal guarantees empirically. Oracle-IWMR achieves perfect IoU~(1.00), confirming the theoretical ceiling. USU outperforms all interpolation baselines: on RRL-CNN, IoU rises from 0.79 to~0.86 and pointing game from 78\% to 100\%, with concentration improving $0.54{\to}0.76$ as USU amplifies proper attribution priors. The USU--Oracle gap (0.86 vs.\ 0.96) quantifies headroom from imperfect segment scoring. Qualitatively, USU-IWMR recovers distinct attribution patterns (\Cref{fig:results-hero}, top) where interpolation collapses them (\Cref{fig:identical-saliencies}). \Cref{tab:desiderata} validates the formal desiderata: USU-Fixed achieves (D1) conservation at numerical precision ($4.95 \times 10^{-7}$), three to four orders below interpolation, consistent with \Cref{thm:interp-d1,thm:interp-d2,thm:interp-d4}.
\begin{table*}[t]
\centering
\begin{minipage}[t]{0.48\linewidth}
\centering
\caption{\textbf{Controlled validation: shape classification.} IoU ($\uparrow$), Conc.\ ($\uparrow$), PG ($\uparrow$). Bold is best non-oracle; oracles use ground-truth scores.}
\label{tab:synthetic-results}
\small
\begin{tabular}{@{}llccc@{}}
\toprule
Model & Method & IoU & Conc. & PG \\
\midrule
\multirow{5}{*}{RRL-CNN}
& Bilinear       & 0.79 & 0.54 & 0.78 \\
& Bicubic        & 0.82 & 0.59 & 0.79 \\
& USU            & \textbf{0.86} & \textbf{0.76} & \textbf{1.00} \\
& Oracle-USU     & 0.96 & 0.88 & 1.00 \\
& Oracle-IWMR    & 1.00 & 0.85 & 1.00 \\
\bottomrule
\end{tabular}
\end{minipage}%
\hfill
\begin{minipage}[t]{0.48\linewidth}
\centering
\caption{\textbf{Desiderata verification on ImageNet + VGG16.} \ding{51} = satisfied; \ding{55} = violated. D1 Error is mean absolute neighbourhood mass discrepancy.}
\label{tab:desiderata}
\small
\begin{tabular}{@{}lccccr@{}}
\toprule
Method & D1 & D2 & D3 & D4 & D1 Error \\
\midrule
Bilinear   & \ding{55} & \ding{55} & \ding{55} & \ding{55} & $1.23 \!\times\! 10^{-1}$ \\
Bicubic    & \ding{55} & \ding{55} & \ding{55} & \ding{55} & $8.65 \!\times\! 10^{-1}$ \\
Lanczos-3  & \ding{55} & \ding{55} & \ding{55} & \ding{55} & $2.14 \!\times\! 10^{-2}$ \\
USU-Fixed  & \ding{51} & \ding{51} & \ding{51} & \ding{51} & $4.95 \!\times\! 10^{-7}$ \\
USU-IWMR   & \ding{55} & \ding{51} & \ding{51} & \ding{51} & $3.81 \!\times\! 10^{-4}$ \\
\bottomrule
\end{tabular}
\end{minipage}
\end{table*}
\subsection{In-the-Wild Faithfulness Evaluation}
\label{sec:faithfulness}
\paragraph{Setup.}
We sample 1{,}000 images each from ImageNet~\cite{7ca07ef2c47cccda38db8f2c1f87f64a51c70b11}, CIFAR-10~\cite{krizhevsky2009learning}, and CUB-200~\cite{wah2011caltech}, evaluated on VGG16~\cite{3ce9cb15eb8f0ea8b3f8e4461c38bf3f95f51f5e}, ResNet50~\cite{b49f7bbc2b3c78dc2bc2f67eb09a6e63aaab0efa}, and ViT-B-16~\cite{e5bb4fcb0aeb7d9de6b5deac12fa41efa88a81aa}. Since USU's guarantees depend only on the tuple $(M_k, s, \cN, \cS)$ (\Cref{thm:uniqueness}), results are consistent across attribution methods; we report GradCAM as representative (seven-method comparison in the supplemental material). Primary metric: Infidelity~$\downarrow$ from Quantus~\cite{30e776268268e84becd2863b0632247da61238b9}; additional metrics in the supplemental material.
\vspace{-1.5ex}
\paragraph{Results.}
\Cref{tab:main-results} reports Infidelity across all nine dataset--model combinations. USU reduces Infidelity by one to four orders of magnitude over bilinear (Wilcoxon signed-rank $p < 0.001$, Bonferroni-corrected). Gains are largest on ViT-B-16, whose $14{\times}14$ patch structure aligns naturally with segment boundaries, and smallest on CUB-200 + ResNet50, where fine-grained texture classes limit segmentation-based redistribution. These magnitudes reflect Infidelity's quadratic sensitivity to systematic mass violations: interpolation injects coherent boundary errors at every neighbourhood (\Cref{thm:interp-d1}), which accumulate across hundreds of boundary-crossing regions; USU eliminates this entire error class by construction (D1, verified to $10^{-7}$ precision; \Cref{tab:desiderata}), leaving only residual segmentation noise. USU-IWMR outperforms USU-Fixed on CIFAR-10 and CUB-200 where coarse attributions exhibit background bias ($\varepsilon{=}0.1$ optimal). Full ablations in the supplemental material.
\begin{table*}[t]
\centering
\caption{\textbf{Infidelity ($\downarrow$) across datasets and models.} USU achieves one to four orders of magnitude improvement over interpolation baselines. Best in \textbf{bold}, second-best \underline{underlined}.}
\label{tab:main-results}
\footnotesize
\renewcommand{\arraystretch}{1.15}
\begin{tabularx}{\textwidth}{@{}llCCC@{\quad}|@{\quad}CC@{}}
\toprule
\textbf{Dataset} & \textbf{Model} & \textbf{Bilinear} & \textbf{Bicubic} & \textbf{Lanczos} & \textbf{USU} & \textbf{USU$_{\text{IWMR}}$} \\
\midrule
\rowcolor{black!6}
  & VGG16     & $6.91 \times 10^{6}$ & $7.82 \times 10^{6}$ & $8.14 \times 10^{6}$ & $\mathbf{1.14 \times 10^{5}}$ & $\underline{1.28 \times 10^{5}}$ \\
\rowcolor{black!6}
ImageNet & ResNet50  & $4.87 \times 10^{7}$ & $5.41 \times 10^{7}$ & $5.63 \times 10^{7}$ & $\mathbf{2.21 \times 10^{5}}$ & $\underline{2.45 \times 10^{5}}$ \\
\rowcolor{black!6}
  & ViT-B-16  & $2.19 \times 10^{8}$ & $2.48 \times 10^{8}$ & $2.59 \times 10^{8}$ & $\mathbf{7.02 \times 10^{5}}$ & $\underline{7.85 \times 10^{5}}$ \\
\midrule
  & VGG16     & $3.12 \times 10^{7}$ & $3.55 \times 10^{7}$ & $3.71 \times 10^{7}$ & $\underline{7.12 \times 10^{4}}$ & $\mathbf{6.47 \times 10^{4}}$ \\
CIFAR-10 & ResNet50  & $8.95 \times 10^{7}$ & $9.87 \times 10^{7}$ & $1.02 \times 10^{8}$ & $\mathbf{2.21 \times 10^{5}}$ & $\underline{2.38 \times 10^{5}}$ \\
  & ViT-B-16  & $2.57 \times 10^{8}$ & $2.89 \times 10^{8}$ & $3.01 \times 10^{8}$ & $\underline{2.15 \times 10^{4}}$ & $\mathbf{1.99 \times 10^{4}}$ \\
\midrule
\rowcolor{black!6}
  & VGG16     & $5.44 \times 10^{6}$ & $6.12 \times 10^{6}$ & $6.38 \times 10^{6}$ & $\underline{1.82 \times 10^{3}}$ & $\mathbf{1.70 \times 10^{3}}$ \\
\rowcolor{black!6}
CUB-200 & ResNet50  & $8.46 \times 10^{7}$ & $9.32 \times 10^{7}$ & $9.68 \times 10^{7}$ & $\underline{1.52 \times 10^{7}}$ & $\mathbf{1.41 \times 10^{7}}$ \\
\rowcolor{black!6}
  & ViT-B-16  & $1.83 \times 10^{8}$ & $2.05 \times 10^{8}$ & $2.14 \times 10^{8}$ & $\mathbf{5.08 \times 10^{4}}$ & $\underline{5.61 \times 10^{4}}$ \\
\bottomrule
\end{tabularx}
\end{table*}

\section{Conclusion}
\label{sec:conclusion}

Attribution methods still upsample saliency maps with interpolation techniques designed for natural images, not saliency maps. We have shown that faithful upsampling is not this solved implementation detail but a theoretically constrained operation that admits a unique ratio-form solution. Four desiderata formalize what faithful upsampling requires; standard interpolation structurally violates three, producing artifacts that distort what practitioners see and corrupt what faithfulness metrics measure. USU emerges as the unique ratio-form operator satisfying all four, with uniqueness forced by Luce's Independence of Irrelevant Alternatives. USU is a method-agnostic drop-in replacement for interpolation, delivering one to four orders-of-magnitude faithfulness gains across three datasets, three architectures, and seven attribution methods.
\paragraph{Limitations and future work.}
The uniqueness characterization holds for operators whose allocation weights are independent of the mass being distributed (linearity in mass), the standard redistribution assumption and the direct analogue of Shapley linearity. Extending the characterization to mass-dependent weighting is an open direction.

\bibliographystyle{splncs04}
\bibliography{bibliography/consolidated}

\clearpage
\appendix
\begin{center}
\rule{\textwidth}{1.2pt}\\[10pt]%
{\normalsize\textsc{Supplementary Material}}\\[6pt]%
{\Large\bfseries Attribution Upsampling should Redistribute, Not Interpolate}\\[10pt]%
\rule[0.5ex]{0.3\textwidth}{0.3pt}%
\;\;{\normalsize$\diamond$}\;\;%
\rule[0.5ex]{0.3\textwidth}{0.3pt}%
\end{center}
\bigskip

\paragraph{Overview.}
This supplement delivers: (i)~complete proofs of all theorems deferred from the main paper, (ii)~an extended discussion situating USU within axiomatic attribution and choice theory, (iii)~algorithmic details for hierarchical boundary refinement, and (iv)~comprehensive empirical results across nine configurations spanning three datasets and three models, with seven attribution methods. All proofs are self-contained and verify desiderata satisfaction (D1-D4), global conservation, merge conservation, and full pipeline completeness.

\smallskip\noindent
\textbf{Roadmap.}
\Cref{app:deferred-proofs} delivers the proofs deferred from the main paper: interpolation violations, the uniqueness theorem, Soft IWMR properties, score potential analysis, and combined desiderata verification.
\Cref{app:extended-theory} provides extended theoretical context, connecting USU to axiomatic attribution, resolution gaps in existing methods, and the choice-theoretic foundations that motivate the ratio form.
\Cref{app:depth-fusion} formalizes hierarchical boundary refinement: definitions, H-map, boundary operator, adaptive mixing, merge, and mass conservation through recursion.
\Cref{app:error-analysis} develops the $\alpha$/$\beta$-error decomposition.
\Cref{app:synthetic} details the synthetic experiments with full results across all methods and model families.
\Cref{app:metrics} defines the six Quantus evaluation metrics.
\Cref{app:statistics} describes the statistical methodology (Wilcoxon tests, effect sizes, confidence intervals).
\Cref{app:cross-method} demonstrates attribution-method independence across seven CAM-family sources.
\Cref{app:full-results} provides complete results for ImageNet, CIFAR-10, and CUB-200.

\section{Deferred Proofs}
\label{app:deferred-proofs}

This section delivers the proofs deferred from the main paper, organized to mirror the layered characterization: we first show interpolation fails (\Cref{app:interp-proofs}), then establish the ratio form's uniqueness (\Cref{app:uniqueness-proof}), analyze relaxations (\Cref{app:iwmr-proofs}), identify the tensor potential (\Cref{app:potential-analysis}), and verify combined satisfaction (\Cref{app:desiderata-verification}).

\subsection{Interpolation Violation Proofs}
\label{app:interp-proofs}

The main paper identifies three structural violations of the upsampling desiderata by standard interpolation (\Cref{thm:interp-d1,thm:interp-d2,thm:interp-d4}). We present the full proofs here, organized by causal depth. The root cause is \emph{semantic blindness}: interpolation weights depend solely on spatial coordinates, making the operator structurally incapable of incorporating the model's reasoning. Two geometric consequences follow: mass leakage across neighbourhood boundaries and non-local dependencies at boundary pixels, completing the case that interpolation is fundamentally incompatible with faithful attribution upsampling.

\begin{structframe}{Root cause: Semantic blindness.}
Interpolation treats attribution maps as band-limited signals to be smoothly reconstructed. The interpolated value $\tA(x) = \sum_z K(x,z) \cdot A(z)$ is defined entirely by the kernel~$K$, which encodes spatial proximity, not the model's semantic relevance. The segment scores~$s$, the neighbourhood partition~$\cN$, and the segmentation~$\cS$ are absent from the computation. This absence is not an oversight; it is intrinsic to the interpolation paradigm.
\end{structframe}

\begin{proof}[Proof of \Cref{thm:interp-d2} (Interpolation Ignores Semantics)]
\emph{Score independence.}\enspace For any two segment score functions $s_1$ and $s_2$, the interpolated output is identical:
\[
  \tA(x) = \sum_z K(x,z) \cdot A(z),
\]
since $s$ does not appear in the computation. The output depends only on~$A$ and the kernel~$K$.

\smallskip\noindent
\emph{Monotonicity violation.}\enspace It suffices to exhibit an attribution~$A$ and pixels $x, y \in N_k$ in distinct segments with $\tA(x) \neq \tA(y)$ and $M_k \geq 0$. Any kernel whose weights vary with position (bilinear, bicubic, and all higher-order variants) produces such non-uniform outputs within a neighbourhood. Suppose $\tA(x) > \tA(y)$. Since interpolated values are score-independent, assigning $s(\sidx(y)) > s(\sidx(x))$ does not change the output, yet (D2) requires $\tA(y) \geq \tA(x)$, a contradiction.
\end{proof}

Score independence is the defining limitation: no interpolation kernel, regardless of shape, order, or support size, can incorporate the model's reasoning into the upsampling.

\begin{consequenceframe}{Geometric consequence I: Mass leakage.}
Semantic blindness means interpolation also ignores the neighbourhood partition~$\cN$. Since kernels are designed for signal reconstruction, not mass redistribution, their support routinely spans neighbourhood boundaries, causing attribution mass to leak between neighbourhoods.
\end{consequenceframe}

\begin{proof}[Proof of \Cref{thm:interp-d1} (Mass Conservation Violation)]
Let $x_0 \in N_k$ be a boundary pixel whose kernel support includes $y \in N_{k'} \setminus N_k$ ($k' \neq k$) with non-zero weight $w_y = K(x_0, y) \neq 0$, where no other pixel in~$N_k$ has~$y$ in its support. Construct the indicator attribution $A = \mathbf{1}_{\{y\}}$, so $M_k = 0$. Since only~$x_0$ receives a contribution from~$y$:
\[
  \sum_{x \in N_k} \tA(x) = K(x_0, y) = w_y.
\]
Since $w_y \neq 0$ while $M_k = 0$, mass conservation is violated.
\end{proof}

Mass leakage is not a rare corner case: it occurs at every neighbourhood boundary where the kernel support spans the partition.

\begin{consequenceframe}{Geometric consequence II: Non-local contamination.}
The same boundary-spanning support that causes mass leakage also violates locality: a pixel's attribution can depend on data from adjacent neighbourhoods.
\end{consequenceframe}

\begin{proof}[Proof of \Cref{thm:interp-d4} (Locality Violation)]
Using the same boundary pixel~$x_0$ and external pixel~$y$ with non-zero weight $w_y = K(x_0, y) \neq 0$, construct two attributions that agree on~$N_k$:
\[
  A_1 \equiv 0, \qquad A_2 = \mathbf{1}_{\{y\}}.
\]
Both vanish on~$N_k$ (since $y \notin N_k$), so a local operator must produce identical outputs for every pixel in~$N_k$. Yet:
\[
  \tA_1(x_0) = 0, \qquad \tA_2(x_0) = w_y \neq 0.
\]
The attribution at~$x_0$ depends on data outside its neighbourhood, violating~(D4).
\end{proof}

\begin{synthesisframe}{}
All three violations trace to a single design choice: \emph{interpolation} reconstructs a spatial signal without reference to attribution semantics or neighbourhood structure. Any kernel whose support spans neighbourhood boundaries (bilinear, bicubic, and all higher-order variants used in practice) exhibits all three violations simultaneously.
\end{synthesisframe}

\subsection{Uniqueness Theorem Proof}
\label{app:uniqueness-proof}

Having established that interpolation structurally violates three desiderata, a natural question arises: does any operator satisfy them all? We now show that, under the standard structural assumption of linearity in mass (the direct spatial analogue of Shapley linearity), the answer is yes and the solution is \emph{unique}. The argument has two conceptual stages: locality induces the Independence of Irrelevant Alternatives (\Cref{lem:iia}), and monotonicity resolves the resulting functional equation into a potential (\Cref{lem:potential}). Conservation then normalizes the potential into the ratio form, leaving no remaining degrees of freedom.

\begin{structframe}{Redistribution as rational choice.}
Each pixel in a neighbourhood competes for attribution mass based on its segment score. Locality (D4) requires that the relative allocation between any two pixels be independent of what other pixels are present. This is precisely Luce's Independence of Irrelevant Alternatives~\cite{f8222304ca201f8d524b3aa270673023334e7ed1}, the foundational axiom of discrete choice theory. From IIA, the machinery of functional equations determines the operator form uniquely.
\end{structframe}

\begin{lemma}[IIA from Locality]
\label{lem:iia}
Let $\cU$ be a linear-in-mass upsampling operator satisfying \emph{(D4)} locality. For any neighbourhood $N_k$ and any $x, y \in N_k$, the weight ratio $w_k(x)/w_k(y)$ depends only on the scores $s(\sidx(x))$ and $s(\sidx(y))$, not on the scores of other pixels in $N_k$.
\end{lemma}

\begin{proof}
By (D4), $\tA(x)$ depends only on the mass and segment scores within $N_k$. Consider two neighbourhoods $N_k$ and $N_{k'}$ that both contain pixels with segment scores $s_a$ and $s_b$ but differ in their remaining pixels. By locality, the ratio $w(s_a)/w(s_b)$ must be identical in both neighbourhoods: changing the other pixels (the ``irrelevant alternatives'') cannot affect the relative allocation between $s_a$ and $s_b$.
\end{proof}

IIA is the structural constraint that connects attribution redistribution to classical choice theory. It ensures that adding or removing pixels from a neighbourhood changes the absolute share each pixel receives (through the normalization denominator) but not the relative share between any two pixels.

\begin{consequenceframe}{From independence to uniqueness.}
IIA constrains the weight ratios severely: the pairwise ratio $R(a,b) = w(s_a)/w(s_b)$ must satisfy transitivity $R(a,b)\cdot R(b,c) = R(a,c)$, which is the multiplicative Cauchy functional equation. Under monotonicity, only one family of solutions survives.
\end{consequenceframe}

\begin{lemma}[Potential Factorization]
\label{lem:potential}
Let $\cU$ be a linear-in-mass upsampling operator satisfying IIA (\Cref{lem:iia}) and \emph{(D2)} monotonicity. There exists a continuous, strictly monotone function $\vphi\colon[0,1]\to\mathbb{R}_{>0}$ such that $w_k(x) \propto \vphi(s(\sidx(x)))$ for all $x \in N_k$.
\end{lemma}

\begin{proof}
Define $R(a,b) = w(s_a)/w(s_b)$ for score values $a,b$. By \Cref{lem:iia}, $R$ is well-defined (independent of neighbourhood composition). IIA implies transitivity:
\[
R(a,b) \cdot R(b,c) = \frac{w(s_a)}{w(s_b)} \cdot \frac{w(s_b)}{w(s_c)} = \frac{w(s_a)}{w(s_c)} = R(a,c).
\]
This is the multiplicative Cauchy functional equation. Under continuity and strict monotonicity from (D2), the unique family of solutions is $R(a,b) = \vphi(a)/\vphi(b)$ for some continuous, strictly monotone $\vphi\colon[0,1]\to\mathbb{R}_{>0}$. Strict positivity follows from the requirement that weights be positive for all score values.
\end{proof}

With the potential in hand, conservation selects the unique normalization.

\begin{proof}[Proof of \Cref{thm:uniqueness}]
By linearity in mass, $\tA(x) = M_k \cdot w_k(x)$ for some weight function $w_k$. By \Cref{lem:potential}, $w_k(x) \propto \vphi(s(\sidx(x)))$. (D1) requires:
\[
\sum_{x \in N_k} \tA(x) = M_k \implies \sum_{x \in N_k} w_k(x) = 1.
\]
Therefore $w_k(x) = \vphi(s(\sidx(x)))/\sum_{y\in N_k}\vphi(s(\sidx(y)))$, yielding the ratio form.
\end{proof}

\begin{synthesisframe}{}
The three axioms eliminate degrees of freedom in sequence: locality forces IIA (the relative weight of any two pixels is context-independent), monotonicity resolves the functional equation into a strictly monotone potential, and conservation normalizes the result. Sufficiency is therefore accompanied by uniqueness: no other linear-in-mass architecture satisfies the desiderata.
\end{synthesisframe}

\begin{figure}[t]
  \centering
  \includegraphics[width=\linewidth]{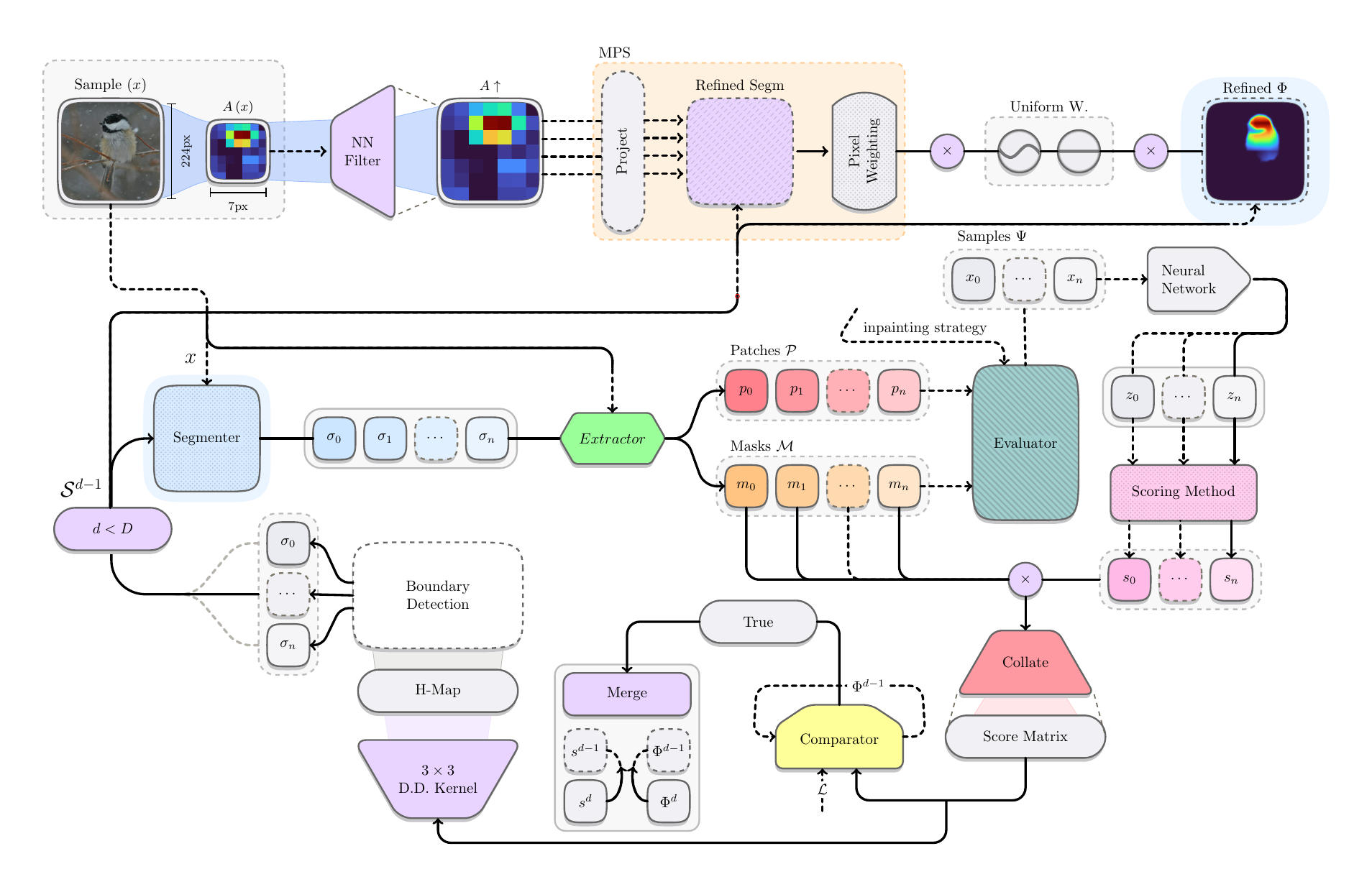}
  \caption{\textbf{Complete USU pipeline.}
    Given input $x$, the pipeline proceeds in three stages.
    \emph{Scoring} (bottom right): masked perturbation evaluates segment importance, producing the score matrix~$\Phi$.
    \emph{Hierarchical refinement} (bottom left): the depth fusion loop convolves $\Phi$ with the diagonal difference kernel to compute the H-map, selects heterogeneous segments for re-segmentation, and merges coarse and fine scores so that segment boundaries emerge naturally from the model's own reasoning; convergence is governed by the comparator threshold~$\mathcal{L}$.
    \emph{Redistribution} (top): the coarse attribution $A(x)$ is projected onto the refined segmentation and redistributed via pixel-level weighting with the tensor potential, yielding~$A^{\uparrow}$.}
  \label{fig:usu-pipeline}
\end{figure}

The complete USU pipeline, from coarse masses through potential-weighted normalization to redistributed attribution, is illustrated in \Cref{fig:usu-pipeline}.

\subsection{Soft IWMR Properties}
\label{app:iwmr-proofs}

Having established uniqueness for strict conservation, we analyze what happens when (D1) is intentionally relaxed to allow importance-weighted mass redistribution across neighbourhoods. The key architectural insight is that IWMR modifies only the mass budget assigned to each neighbourhood while retaining USU's normalized weights unchanged. This separation determines which guarantees survive the relaxation: properties that depend on weights carry over intact, while properties that depend on per-neighbourhood mass budgets are either generalized or intentionally violated.

\begin{structframe}{Design principle: redistribute mass, preserve weights.}
IWMR replaces each neighbourhood's original mass $M_k$ with a redistributed mass $\tilde{M}_k = M_{\mathrm{total}} \cdot \rho_k$, where $\{\rho_k\}$ forms a partition of unity weighted by neighbourhood importance. The normalized weights $w_k(x) \propto \vphi(s(\sidx(x)))$ remain identical to the strict-conservation operator. IWMR and USU therefore differ only in how much mass each neighbourhood receives, not in how that mass is distributed among pixels within it.
\end{structframe}

\begin{theoremframe}
\begin{theorem}[IWMR Properties]
\label{thm:iwmr-properties}
Soft IWMR with the tensor potential satisfies:

\smallskip\noindent
\emph{(D1-G) Generalized completeness.}\enspace $\sum_{x \in N_k} \tA^{\mathrm{IWMR}}(x) = \tilde{M}_k$ for all~$k$;

\smallskip\noindent
\emph{(D2) Segment monotonicity.}\enspace Higher scores yield higher attributions within each neighbourhood when $\tilde{M}_k \geq 0$;

\smallskip\noindent
\emph{(D3) Constant conditioning strength.}\enspace The shared tensor potential ensures uniform score sensitivity;

\smallskip\noindent
\emph{(D4$'$) Semi-locality.}\enspace $\tA^{\mathrm{IWMR}}(x)$ is determined by $\tilde{M}_{\nidx(x)}$ and segment scores within $N_{\nidx(x)}$.

\smallskip
IWMR intentionally violates (D1) whenever $\tilde{M}_k \neq M_k$, which occurs precisely when neighbourhood importances are non-uniform.
\end{theorem}
\end{theoremframe}

\begin{proof}[Proof of (D1-G)]
$\sum_{x \in N_k}\tA^{\mathrm{IWMR}}(x) = \tilde{M}_k\sum_{x \in N_k}w_k(x) = \tilde{M}_k$, since the USU weights normalize to~$1$ by construction.
\end{proof}

Generalized completeness states that each neighbourhood exhausts its redistributed budget exactly. The shift from $M_k$ to $\tilde{M}_k$ replaces per-neighbourhood bookkeeping with a global reallocation, but the accounting within each neighbourhood remains exact.

\begin{proof}[Proof of (D2)]
Within $N_k$, the weights $w_k(x)\propto\vphi(s(\sidx(x)))$ are monotone in score since $\vphi$ is strictly increasing. When $\tilde{M}_k\geq 0$, $\tA^{\mathrm{IWMR}}(x)=\tilde{M}_k\cdot w_k(x)$ preserves the ordering.
\end{proof}

Monotonicity survives because it is a property of the normalized weights, which IWMR inherits from USU without modification. The redistributed mass $\tilde{M}_k$ acts as a non-negative scalar that preserves the weight ordering.

\begin{proof}[Proof of (D3)]
The tensor potential satisfies $\vphi(s{+}\Delta)/\vphi(s)=\exp(\Delta/\varepsilon)$, independent of~$s$ (\Cref{app:potential-analysis}). Since IWMR uses the same potential, conditioning strength is constant.
\end{proof}

\begin{consequenceframe}{From strict locality to semi-locality.}
The three properties above depend only on the weights within each neighbourhood, so they carry over from USU unchanged. Semi-locality is where the architectural separation becomes visible: the redistributed mass $\tilde{M}_k$ depends on all neighbourhood importances through $\rho_k$, injecting a controlled global dependency into an otherwise local computation.
\end{consequenceframe}

\begin{proof}[Proof of (D4$'$)]
Given $\tilde{M}_k$ and the segment scores within $N_k$, $\tA^{\mathrm{IWMR}}(x)=\tilde{M}_k\cdot w_k(x)$ depends on no other data. The distinction from strict (D4) is that $\tilde{M}_k$ itself depends on global importance through~$\rho_k$.
\end{proof}

\begin{proof}[Proof of (D1) violation]
When importances are non-uniform, the redistribution weights $\rho_k$ reflect importance rather than original mass proportions, so $\tilde{M}_k=M_{\mathrm{total}}\cdot\rho_k\neq M_k$ in general.
\end{proof}

The violation is precisely the mechanism by which mass flows toward important neighbourhoods. What strict conservation loses locally, IWMR preserves globally.

\begin{corollary}[Global Conservation and Zero-Sum Flow]
\label{cor:iwmr-conservation}
Since $\{\rho_k\}$ forms a partition of unity, IWMR preserves total mass: $\sum_k \tilde{M}_k = M_{\mathrm{total}}$. Mass flows are zero-sum: $\sum_k (\tilde{M}_k - M_k) = 0$.
\end{corollary}
\begin{proof}
$\sum_k\tilde{M}_k = M_{\mathrm{total}}\sum_k\rho_k = M_{\mathrm{total}}$. Subtracting original masses: $\sum_k(\tilde{M}_k - M_k) = M_{\mathrm{total}} - M_{\mathrm{total}} = 0$.
\end{proof}

\begin{synthesisframe}{}
The weight-preservation architecture of IWMR produces a clean separation of concerns: within-neighbourhood guarantees (monotonicity, conditioning strength, generalized completeness) are inherited from USU intact, while per-neighbourhood mass budgets are globally reallocated via the partition of unity. Locality weakens to semi-locality at a single, identifiable point: the dependence of $\tilde{M}_k$ on global importance. Total mass is conserved; only its distribution across neighbourhoods changes.
\end{synthesisframe}

\subsection{Score Potential Analysis}
\label{app:potential-analysis}

The uniqueness theorem (\Cref{thm:uniqueness}) establishes that any compliant operator takes the ratio form $w_k(x) \propto \vphi(s(\sidx(x)))$ for some strictly monotone potential~$\vphi$, satisfying (D1), (D2), and (D4) simultaneously. The potential itself remains a free parameter: every strictly monotone, positive $\vphi$ yields a valid operator. (D3), constant conditioning strength, is the final selection criterion. It requires that the multiplicative effect of a score gap~$\Delta$ be uniform across the score range, so that the operator treats equal score differences identically regardless of absolute level. We now show that this requirement uniquely selects the exponential family.

\begin{structframe}{Selection principle: score-uniform sensitivity.}
(D3) demands $\vphi(s{+}\Delta)/\vphi(s) = g(\Delta)$ for some function~$g$ independent of~$s$. In log space, this becomes $\log\vphi(s{+}\Delta) - \log\vphi(s) = h(\Delta)$: the increments of $\log\vphi$ must depend only on the step size, not on the starting point. This is the additive Cauchy functional equation, whose continuous solutions are precisely the affine functions $f(s) = cs + d$. Any potential with non-affine logarithm violates (D3).
\end{structframe}

\begin{proposition}[Power-law potentials violate (D3)]
\label{prop:power-law-d3}
For $\vphi(s) = s^\gamma$ with $\gamma > 0$, the conditioning ratio depends on~$s$.
\end{proposition}

\begin{proof}
$\vphi(s{+}\Delta)/\vphi(s) = ((s{+}\Delta)/s)^\gamma$, which varies with~$s$. Equivalently, $\log\vphi(s) = \gamma\log s$ is not affine.
\end{proof}

\begin{proposition}[Log-odds potentials violate (D3)]
\label{prop:log-odds-d3}
For $\vphi(s) = (s/(1{-}s))^\gamma$ with $\gamma > 0$, the conditioning ratio depends on~$s$.
\end{proposition}

\begin{proof}
$\vphi(s{+}\Delta)/\vphi(s) = \bigl((s{+}\Delta)(1{-}s)\big/\bigl((1{-}s{-}\Delta)\,s\bigr)\bigr)^\gamma$, a rational function of~$s$. Equivalently, $\log\vphi(s) = \gamma[\log s - \log(1{-}s)]$ is not affine.
\end{proof}

The pattern is the same in both cases: a non-affine logarithm produces score-dependent sensitivity, amplifying redistribution in some score ranges while suppressing it in others. Since $\log\vphi$ must be affine and $\vphi$ must be strictly monotone and positive, the unique surviving family is $\vphi(s) = Ce^{s/\varepsilon}$ for $C > 0$ and $\varepsilon > 0$, where $C$ cancels in the ratio form and the shift by~$0.5$ is a centering convention.

\begin{theoremframe}
\begin{theorem}[Tensor potential uniquely satisfies (D3)]
\label{thm:tensor-d3}
Among continuous, strictly monotone potentials, $\vphi(s) = \exp((s - 0.5)/\varepsilon)$ is the unique family satisfying constant conditioning strength:
\[
\frac{\vphi(s+\Delta)}{\vphi(s)} = \exp\!\left(\frac{\Delta}{\varepsilon}\right),
\]
independent of~$s$. The temperature~$\varepsilon > 0$ controls redistribution sharpness.
\end{theorem}
\end{theoremframe}

\begin{proof}
(D3) requires $\vphi(s{+}\Delta)/\vphi(s) = g(\Delta)$. Setting $f = \log\vphi$, this becomes $f(s{+}\Delta) - f(s) = h(\Delta)$, the additive Cauchy equation. Under continuity, $f(s) = cs + d$ for constants $c, d$, yielding $\vphi(s) = e^{cs+d} = Ce^{s/\varepsilon}$ with $\varepsilon = 1/c > 0$ (strict monotonicity requires $c > 0$). Direct verification: $\vphi(s{+}\Delta)/\vphi(s) = e^{c\Delta} = \exp(\Delta/\varepsilon)$.
\end{proof}

\begin{synthesisframe}{Operational consequence.}
Violating (D3) introduces \emph{score-dependent redistribution sensitivity}: for $\vphi(s) = s^2$, a gap of $\Delta = 0.1$ amplifies the weight ratio by ${\approx}\,1.27$ near $s = 0.8$ but by $4.0$ near $s = 0.1$, over $3{\times}$ the sensitivity. The tensor potential eliminates this non-uniformity. Practitioners may nonetheless prefer alternative potentials when score distributions cluster in a known range; the ratio form retains (D1), (D2), and (D4) guarantees for any strictly monotone~$\vphi$.
\end{synthesisframe}

\subsection{USU Desiderata Verification}
\label{app:desiderata-verification}

The interpolation violation proofs (\Cref{app:interp-proofs}) established that semantic blindness, mass leakage, and non-local contamination are structural consequences of the interpolation paradigm. We now verify that USU's ratio form resolves all three violations and, with the tensor potential, satisfies all four desiderata simultaneously.

\begin{structframe}{Design principle: Score-aware mass redistribution.}
USU replaces interpolation's fixed spatial kernel with normalized weights $w_k(x) = \vphi(s(\sidx(x)))/\sum_{y \in N_k}\vphi(s(\sidx(y)))$ that encode the model's reasoning and respect the neighbourhood partition by construction. Three algebraic properties of these weights (they sum to one, they inherit the potential's monotonicity, and they depend only on local data) directly resolve the three structural failures.
\end{structframe}

\begin{proof}[Proof of \Cref{thm:usu-all}~(D1): Neighbourhood Completeness]
\begin{align*}
\sum_{x \in N_k} \tA(x) &= \sum_{x \in N_k} M_k \cdot w_k(x) = M_k \sum_{x \in N_k} \frac{\vphi(s(\sidx(x)))}{\sum_{y\in N_k}\vphi(s(\sidx(y)))} = M_k \cdot 1 = M_k.
\end{align*}
\end{proof}

Where interpolation leaks mass across neighbourhood boundaries through its spatially extended kernel, the normalized weights confine redistribution entirely within each neighbourhood. The denominator sums over~$N_k$ alone, so every unit of~$M_k$ is redistributed to pixels in~$N_k$ and nowhere else.

\begin{proof}[Proof of \Cref{thm:usu-all}~(D2): Segment Monotonicity]
For $x,y\in N_k$ with $s(\sidx(x))\geq s(\sidx(y))$: since the exponential is strictly increasing, $\vphi(s(\sidx(x)))\geq\vphi(s(\sidx(y)))$, hence $w_k(x)\geq w_k(y)$. When $M_k\geq 0$, $\tA(x)=M_k\cdot w_k(x)\geq M_k\cdot w_k(y)=\tA(y)$.
\end{proof}

Monotonicity is the direct resolution of semantic blindness: the potential~$\vphi$ translates the model's segment scores into weight ordering, ensuring that higher-scored regions receive proportionally more attribution. Where interpolation ignores scores entirely, USU makes them the sole determinant of within-neighbourhood allocation.

\begin{proof}[Proof of \Cref{thm:usu-all}~(D4): Locality]
By construction, $\tA(x)=M_{\nidx(x)}\cdot w_{\nidx(x)}(x)$ where $M_{\nidx(x)}=\sum_{y\in N_{\nidx(x)}}A(y)$ and $w_{\nidx(x)}(x)$ depends only on segment scores within $N_{\nidx(x)}$. No information from other neighbourhoods is required.
\end{proof}

\begin{consequenceframe}{Calibration: the final degree of freedom.}
(D1), (D2), and (D4) hold for any strictly monotone potential~$\vphi$. The remaining freedom, the choice of~$\vphi$ within the monotone family, is resolved by (D3), which demands score-uniform sensitivity. As shown in \Cref{app:potential-analysis}, only the exponential family satisfies this requirement.
\end{consequenceframe}

\begin{proof}[Proof of \Cref{thm:usu-all}~(D3): Constant Conditioning Strength]
Verified in \Cref{app:potential-analysis}: $\vphi(s{+}\Delta)/\vphi(s)=\exp(\Delta/\varepsilon)$, independent of~$s$.
\end{proof}

\begin{synthesisframe}{}
USU's normalized weights resolve the three structural failures of interpolation in exact correspondence: weight normalization eliminates mass leakage~(D1), potential monotonicity incorporates score semantics~(D2), and neighbourhood-local computation prevents non-local contamination~(D4). The tensor potential then selects uniform sensitivity across the score range~(D3). Together, the four properties confirm that USU is a complete remedy for the interpolation paradigm's incompatibilities with faithful attribution upsampling.
\end{synthesisframe}

\section{Extended Theoretical Context}
\label{app:extended-theory}

The main paper introduces USU's desiderata and derives the ratio form; here we situate this contribution within the broader landscape of axiomatic attribution theory, resolution gaps in existing methods, and the choice-theoretic functional equations that underpin the uniqueness argument.

\subsection{Axiomatic Attribution Theory}
\label{app:axiomatic-context}

USU extends the axiomatic tradition of feature attribution to the upsampling domain. In the Shapley framework, the \emph{efficiency} axiom requires that player contributions sum to the coalition value~\cite{442e10a3c6640ded9408622005e3c2a8906ce4c2}; in Integrated Gradients, \emph{completeness} requires that attributions sum to the prediction difference~\cite{f302e136c41db5de1d624412f68c9174cf7ae8be}. Desideratum (D1), neighbourhood completeness, is the spatial analogue: pixel attributions within each neighbourhood must sum to the neighbourhood mass. While Shapley efficiency and IG completeness operate on feature-level or path-integral attributions, (D1) operates on spatial redistribution, ensuring that the upsampling step does not introduce or destroy attribution mass.

Domain-specific axiomatization is required by the impossibility results that follow. Bilodeau~\etal~\cite{8102cad29ec8e808c7395ac6ee668da495f07206} establish impossibility results showing that no single attribution method can simultaneously satisfy all desirable properties in full generality; different domains require tailored axiom sets. USU addresses this by restricting scope to upsampling, a well-defined subproblem where a complete, satisfiable axiomatization is achievable.

\subsection{Resolution Gaps in Attribution Methods}
\label{app:resolution-gaps}

The resolution gap between model reasoning and explanation granularity arises from distinct architectural sources. For CAM-family methods, attributions are defined on the final convolutional feature map (typically $7{\times}7$ or $14{\times}14$) and upsampled to input resolution. For Vision Transformers, patch tokenization (typically $16{\times}16$) creates attributions at patch granularity. Perturbation-based methods (RISE, SHAP) operate on coarse masks or superpixels, inheriting their resolution.

All existing refinement strategies for closing this gap, including multi-layer CAM fusion (LayerCAM, CAMERAS), augmentation-based approaches (Augmented Grad-CAM), and propagation schemes, are heuristic: none formally characterizes what faithful upsampling requires. A multi-layer fusion approach may sharpen boundaries but cannot state, let alone guarantee, the properties an upsampling operator should satisfy. Crucially, these strategies are integral components of specific XAI methods, not standalone operators. USU fills this gap on two levels: it provides the first axiomatic framework defining what faithful attribution upsampling means (conservation, monotonicity, calibration, locality), and it delivers a method-agnostic operator, composable with any coarse attribution method, that provably satisfies all four desiderata, verified to $10^{-7}$ numerical precision (\Cref{tab:desiderata} in the main paper).

\subsection{Choice Theory and Functional Equations}
\label{app:choice-theory}

The uniqueness proof (\Cref{thm:uniqueness}) draws on a deep connection to discrete choice theory that deserves explicit elaboration. Luce's Choice Axiom~\cite{f8222304ca201f8d524b3aa270673023334e7ed1} states that the relative probability of choosing alternative $a$ over $b$ is independent of what other alternatives are available, precisely the IIA (Independence of Irrelevant Alternatives) property that (D4) locality induces in USU's weight ratios. McFadden's conditional logit~\cite{ea84a6ef34223f4f0d8b64555a6b6cec312b8fce} shows that IIA combined with regularity conditions yields the multinomial logit form $P(a \mid B) \propto \exp(\theta \cdot a)$. Breitmoser~\cite{64cce5f2fb89e06be57b32294df1078dd8227bad} establishes that IIA, positivity, and translation invariance on attributes are necessary and sufficient for conditional logit, which is precisely USU's ratio form with the tensor potential. The additional constant-conditioning requirement (D3) then selects the exponential potential via the additive Cauchy functional equation, paralleling how translation invariance selects the exponential in choice theory.

USU's redistribution problem \emph{is} a discrete choice problem. Each pixel in a neighbourhood ``competes'' for attribution mass based on its segment score, and the desiderata encode the same rationality conditions (consistency via IIA/locality, efficiency via conservation, and sensitivity via monotonicity) that characterize rational stochastic choice~\cite{251646ea79c2a3973311852975f70a64878fffec,f10afb5d088114389a71148d39aab0c57820b02e}.

\begin{figure}[t]
  \centering
  \includegraphics[width=\linewidth]{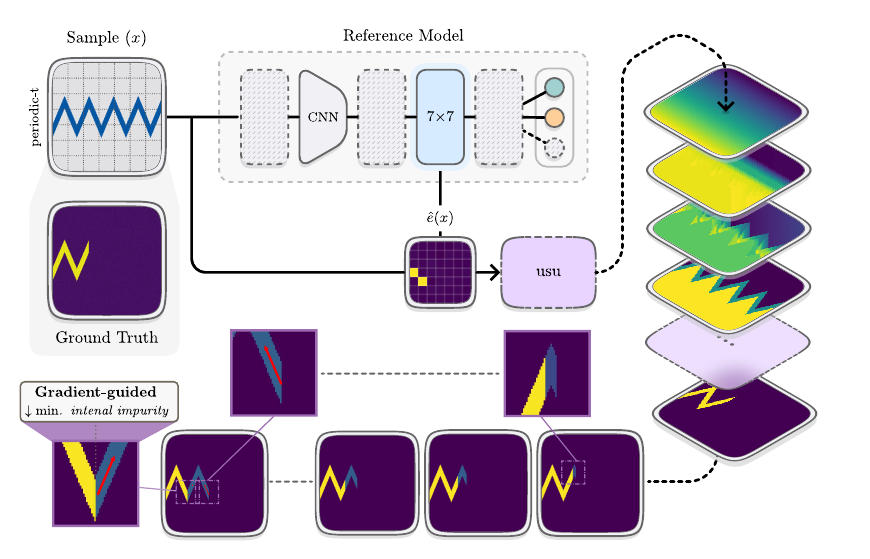}
  \caption{\textbf{Hierarchical boundary refinement.} At each depth level, the score matrix $\Phi^d$ captures per-segment importance, the H-map detects score transitions at segment boundaries, and the adaptive merge combines coarse and fine attributions. Only boundary segments (high $|H_{\mathrm{map}}|$) are refined further, analogous to splitting impure nodes in a decision tree (\Cref{tab:tree-analogy}).}
  \label{fig:hierarchical-refinement}
\end{figure}

\section{Emergent Boundaries from Score Heterogeneity}
\label{app:depth-fusion}

\Cref{app:deferred-proofs} assumed a fixed segmentation partition. In practice, no single segmentation scale suffices: coarse partitions group semantically distinct regions, while fine partitions fragment coherent ones. This section formalizes how score heterogeneity itself selects boundary resolution, illustrated in \Cref{fig:hierarchical-refinement}.

The ratio form (\Cref{sec:uniqueness}) determines how mass is allocated within each neighbourhood, but the segmentation partition $\cS$ has so far been assumed, not derived. For faithful explanations, boundaries must emerge from the model's own reasoning, refined where scores transition and left coarse where scores are homogeneous. No single segmentation scale achieves this: a coarse partition groups semantically distinct regions; a fine partition fragments coherent ones. The key insight is to let \emph{score heterogeneity itself} select boundary resolution. We now formalize the refinement pipeline, from geometric primitives through adaptive boundary selection to the capstone guarantee: total attribution mass is preserved through the full depth recursion.

\begin{definition}[Hierarchical Segmentation]
\label{def:hier-seg}
A \emph{hierarchical segmentation} of depth $D$ is a family $\{\cS^0,\ldots,\cS^{D-1}\}$ of partitions of $\Omega$ with $P_0 \leq \cdots \leq P_{D-1}$ segments, satisfying the refinement predicate: every segment at depth $d{+}1$ is contained in some segment at depth $d$. Refinement is reflexive and transitive.
\end{definition}

Given depth-indexed scores $s^d\colon\{0,\ldots,P_d{-}1\}\to[0,1]$, the \emph{score matrix} at depth $d$ maps each pixel to the score of its containing segment:
\begin{equation}
  \Phi^d(x) = s^d(\sidx^d(x)).
  \label{eq:score-matrix}
\end{equation}

\paragraph{Detecting importance transitions.}
Convolving $\Phi^d$ with a $3{\times}3$ diagonal difference kernel $K_{\mathrm{DD}}$ yields the \emph{H-map}:
\begin{equation}
  H_{\mathrm{map}}(x) = (\Phi^d * K_{\mathrm{DD}})(x),
  \label{eq:hmap}
\end{equation}
measuring local score heterogeneity. When $\Phi^d$ is constant in the $3{\times}3$ neighbourhood of an interior pixel, $H_{\mathrm{map}}(x)=0$: no finer boundary is needed. Segments whose maximum $|H_{\mathrm{map}}|$ exceeds threshold $\theta>0$ are \emph{boundary segments} and are refined at the next depth; homogeneous segments are retained. Lowering $\theta$ weakly increases the boundary set (monotonicity; \Cref{prop:boundary-mono}).

\paragraph{Adaptive depth fusion.}
A sigmoid-based mixing coefficient $\alpha(x)$, monotonically decreasing in $|H_{\mathrm{map}}(x)|$, blends coarse and fine depths: near boundaries ($|H_{\mathrm{map}}| \gg \mu$), fine scores dominate; away from them, coarse scores are retained. The convex combination preserves $[0,1]$ boundedness. Recursion terminates when successive score matrices converge in Frobenius norm. Since each recursion step applies USU (\Cref{thm:usu-all}), neighbourhood completeness propagates inductively and the full pipeline preserves total attribution mass (\Cref{thm:mass-conservation-recursion}).

\begin{table*}[t]
\centering
\caption{\textbf{Decision-tree analogy.} Each stage of USU's hierarchical refinement mirrors a decision-tree operation, with a corresponding formal guarantee verified in this supplement.}
\label{tab:tree-analogy}
\small
\newcolumntype{R}{>{\raggedright\arraybackslash}X}
\begin{tabularx}{\textwidth}{@{}l@{\hspace{8pt}}R@{\hspace{8pt}}R@{}}
\toprule
\textbf{Decision Tree} & \textbf{USU Refinement} & \textbf{Formal Guarantee} \\
\midrule
Root node       & Coarsest partition $\cS^0$ over $\Omega$ & Partitions pixel grid; refinement is reflexive and transitive (\Cref{def:hier-seg}) \\
\addlinespace[5pt]
Node impurity   & Score heterogeneity $\max_{x \in \sigma_p}|H_{\mathrm{map}}(x)|$ & Linear in $\Phi^d$; zero on constant regions (\Cref{prop:hmap-linear,prop:hmap-zero}) \\
\addlinespace[5pt]
Split criterion & Boundary threshold $|H_{\mathrm{map}}| > \theta$ & Monotone in $\theta$: lowering $\theta$ weakly increases boundary set (\Cref{prop:boundary-mono}) \\
\addlinespace[5pt]
Child nodes     & Finer sub-segments $\sigma_q^{d+1} \subseteq \sigma_p^d$ & Refinement predicate: every child contained in parent (\Cref{def:hier-seg}) \\
\addlinespace[3pt]
\midrule
\addlinespace[3pt]
Leaf value      & Merged score $\Phi_{\mathrm{merged}}(x)$ via adaptive $\alpha(x)$ & Convex combination; bounded in $[0,1]$; identity on equal inputs (\Cref{prop:merge-convex,prop:merge-bounded,prop:merge-identity}) \\
\addlinespace[5pt]
Stopping rule   & $\|\Phi^d {-} \Phi^{d-1}\|_F \leq \mathcal{L}$ or $d = D$ & Symmetric and self-terminating (\Cref{prop:comparator-symm,prop:comparator-self}) \\
\addlinespace[5pt]
Prediction      & Fused attribution $\tA(x)$ & Total mass conserved through full recursion (\Cref{thm:mass-conservation-recursion}) \\
\bottomrule
\end{tabularx}
\end{table*}

The remainder of this section formalizes each component of the hierarchical refinement pipeline. We begin with the geometric foundations for boundary detection (\Cref{app:boundary-defs}), then formalize the H-map that detects score transitions (\Cref{app:hmap-proofs}), the boundary operator that selects segments for refinement (\Cref{app:boundary-mono}), the stopping criterion (\Cref{app:comparator-props}), the adaptive mixing coefficient that blends coarse and fine depths (\Cref{app:adaptive-mixing}), the merge operator (\Cref{app:merge-proofs}), and finally the depth recursion with its mass conservation guarantee (\Cref{app:mass-conservation}).

\begin{structframe}{Detection principle: score heterogeneity as boundary signal.}
The H-map, obtained by convolving the score matrix with a diagonal difference kernel, drives both boundary detection and adaptive depth mixing. Where scores are locally constant, the H-map vanishes (\Cref{prop:hmap-zero}) and no finer resolution is needed; at segment boundaries with differing scores, the H-map registers the transition and triggers refinement at the next depth. This dichotomy propagates through the pipeline: it determines which segments split, how mixing coefficients weight coarse against fine depths, and ultimately when refinement converges.
\end{structframe}

\subsection{Boundary Definitions}
\label{app:boundary-defs}

The geometric primitives below formalize what it means for segments to share a boundary, analogous to defining adjacency between nodes in the decision-tree analogy (\Cref{tab:tree-analogy}).

\begin{definition}[Pixel Adjacency]
\label{def:adjacency}
Pixels $x,y \in \Omega$ are \emph{adjacent} iff $\|x - y\|_\infty = 1$, where $\|x-y\|_\infty = \max(|x_1-y_1|, |x_2-y_2|)$ (8-connectivity). Adjacency is symmetric and irreflexive.
\end{definition}

\begin{definition}[Segment Boundary]
\label{def:seg-boundary}
The \emph{boundary} of segment $\sigma_p$ is $\partial\sigma_p = \{x \in \sigma_p : \exists\, y \notin \sigma_p,\, \|x-y\|_\infty = 1\}$. The \emph{interior} is $\sigma_p \setminus \partial\sigma_p$; boundary and interior partition the segment: $\sigma_p = \partial\sigma_p \cup (\sigma_p \setminus \partial\sigma_p)$.
\end{definition}

\begin{definition}[Inter-Segment Boundary]
\label{def:inter-seg-boundary}
The \emph{inter-segment boundary} between $\sigma_p$ and $\sigma_q$ is $\partial(\sigma_p, \sigma_q) = \{(x,y) : x \in \sigma_p,\, y \in \sigma_q,\, \|x-y\|_\infty = 1\}$. This relation is symmetric: $(x,y) \in \partial(\sigma_p, \sigma_q)$ iff $(y,x) \in \partial(\sigma_q, \sigma_p)$.
\end{definition}

These geometric primitives formalize the spatial relations that the H-map probes: boundary detection reduces to testing whether adjacent pixels carry different segment scores, a relation fully characterized by pixel adjacency and segment membership.

\subsection{H-Map Properties}
\label{app:hmap-proofs}

The H-map plays the role of the impurity measure in the decision-tree analogy: it quantifies local score heterogeneity, determining which segments require finer partitioning.

\begin{proposition}[Linearity]
\label{prop:hmap-linear}
The H-map is linear: $H_{\mathrm{map}}(\Phi_1 + \Phi_2) = H_{\mathrm{map}}(\Phi_1) + H_{\mathrm{map}}(\Phi_2)$ and $H_{\mathrm{map}}(r\cdot\Phi) = r\cdot H_{\mathrm{map}}(\Phi)$ for $r \in \mathbb{R}$.
\end{proposition}
\begin{proof}
Linearity of 2D convolution: $(f+g)*K = f*K + g*K$ and $(rf)*K = r(f*K)$.
\end{proof}

\begin{proposition}[Zero on Constant Regions]
\label{prop:hmap-zero}
For an interior pixel $x$ (all corners of its $3{\times}3$ window lie in $\Omega$), if $\Phi^d$ is constant on the $3{\times}3$ neighbourhood, then $H_{\mathrm{map}}(x) = 0$.
\end{proposition}
\begin{proof}
If $\Phi^d(y) = c$ for all $y$ with $\|x-y\|_\infty \leq 1$, then $H_{\mathrm{map}}(x) = c\sum_{i,j}K_{\mathrm{DD}}[i,j] = c(-1+0+1+0+0+0+1+0-1) = 0$.
\end{proof}

Together, linearity and the zero-on-constant property establish the H-map as a selective heterogeneity detector: it responds to genuine score transitions, not to uniform shifts in the score field. This selectivity is what allows the boundary operator to distinguish segments that genuinely require refinement from those that are merely high- or low-scored.

\begin{consequenceframe}{From detection to adaptive partitioning.}
The H-map provides a continuous heterogeneity signal at every pixel. Two discrete decisions remain: which segments to refine (the boundary operator's threshold test) and when to stop refining (the comparator's convergence criterion). Both decisions inherit the H-map's selectivity: only genuine score transitions trigger further partitioning.
\end{consequenceframe}

\subsection{Boundary Operator}
\label{app:boundary-mono}

The boundary operator implements the split criterion: it partitions segments into those requiring refinement (high impurity) and those retained at the current depth (homogeneous).

\begin{definition}[Boundary Operator]
\label{def:boundary-op}
A segment $\sigma_p^d$ is a \emph{boundary segment} iff $\max_{x \in \sigma_p^d} |H_{\mathrm{map}}(x)| > \theta$ for threshold $\theta > 0$. The boundary set is $\mathcal{B}(H_{\mathrm{map}}) = \{\sigma_p^d : \max_{x \in \sigma_p^d} |H_{\mathrm{map}}(x)| > \theta\}$. Non-boundary segments are retained at the current depth; boundary segments are refined at depth $d{+}1$.
\end{definition}

\begin{proposition}[Monotonicity in Threshold]
\label{prop:boundary-mono}
For thresholds $\theta_1 \leq \theta_2$, $\mathcal{B}_{\theta_2}(H_{\mathrm{map}}) \subseteq \mathcal{B}_{\theta_1}(H_{\mathrm{map}})$.
\end{proposition}
\begin{proof}
If $\max_{x \in \sigma_p}|H_{\mathrm{map}}(x)| > \theta_2 \geq \theta_1$, then $\sigma_p \in \mathcal{B}_{\theta_1}$.
\end{proof}

Monotonicity in threshold gives practitioners a single interpretable control: lowering $\theta$ weakly increases the boundary set, producing finer partitions at the cost of additional computation. As $\theta \to 0^+$, every segment with nonzero score variation is refined; as $\theta \to \infty$, no segment is refined and the coarsest partition is retained.

\subsection{Comparator Properties}
\label{app:comparator-props}

The stopping rule determines when further refinement yields diminishing returns, analogous to pruning in decision trees.

\paragraph{Stopping criterion.}
Recursion terminates when score matrices stabilize. The \emph{global comparator} $C(\Phi^d, \Phi^{d-1}) = \mathbf{1}[\|\Phi^d - \Phi^{d-1}\|_F > \mathcal{L}]$ returns False when refinement has converged ($\mathcal{L}>0$).

\begin{proposition}
\label{prop:comparator-symm}
The global comparator is symmetric: $C(\Phi^d, \Phi^{d-1}) = C(\Phi^{d-1}, \Phi^d)$.
\end{proposition}
\begin{proof}
$\|\Phi^d - \Phi^{d-1}\|_F = \|\Phi^{d-1} - \Phi^d\|_F$ by $\|{-}v\| = \|v\|$.
\end{proof}

\begin{proposition}
\label{prop:comparator-self}
$C(\Phi, \Phi) = \mathrm{False}$ for any score matrix $\Phi$.
\end{proposition}
\begin{proof}
$\|\Phi - \Phi\|_F = 0 \not> \mathcal{L}$ since $\mathcal{L} > 0$.
\end{proof}

Symmetry ensures the stopping criterion is direction-independent, while self-termination guarantees that recursion halts once the score matrix stabilizes, providing a natural convergence condition without requiring an external depth budget.

\subsection{Adaptive Mixing Coefficient}
\label{app:adaptive-mixing}

Rather than a hard split between coarse and fine depths, the adaptive mixing coefficient provides a smooth transition controlled by local boundary strength.

The \emph{adaptive mixing coefficient}
\begin{equation}
  \alpha(x) = \sigma\!\left(\frac{-(|H_{\mathrm{map}}(x)| - \mu)}{\tau}\right),
  \label{eq:adaptive-alpha}
\end{equation}
with centering $\mu$ and temperature $\tau > 0$, satisfies:
\begin{enumerate}
\item $\alpha(x) \in (0,1)$ for all $x$ (sigmoid range on $\mathbb{R}$).
\item $\alpha$ is strictly decreasing in $|H_{\mathrm{map}}(x)|$: higher boundary strength yields lower $\alpha$ (favouring finer depth).
\item When $|H_{\mathrm{map}}(x)| = \mu$: $\alpha(x) = 1/2$ (equal weighting).
\item When $|H_{\mathrm{map}}(x)| > \mu + 5\tau$: $\alpha(x) < 0.01$ (fine depth dominates).
\item When $|H_{\mathrm{map}}(x)| < \mu - 5\tau$: $\alpha(x) > 0.99$ (coarse depth dominates).
\end{enumerate}

\subsection{Merge Operator}
\label{app:merge-proofs}

The merge operator combines the leaf values from adjacent depth levels, producing the fused score matrix that drives the final attribution.

The merged score matrix blends coarse and fine depth levels:
\begin{equation}
  \Phi_{\mathrm{merged}}(x) = \alpha(x)\cdot\Phi^{d-1}(x) + (1-\alpha(x))\cdot\Phi^d(x).
  \label{eq:merge-op}
\end{equation}
Near boundaries ($|H_{\mathrm{map}}| \gg \mu$), $\alpha \approx 0$ and finer scores dominate; away from boundaries ($|H_{\mathrm{map}}| \ll \mu$), $\alpha \approx 1$ and coarser scores are retained.

\begin{proposition}[Convex Combination]
\label{prop:merge-convex}
Since $\alpha(x) \in (0,1)$ and $1-\alpha(x) \in (0,1)$ with $\alpha(x) + (1-\alpha(x)) = 1$, the merge $\Phi_{\mathrm{merged}}(x) = \alpha(x)\Phi^{d-1}(x) + (1-\alpha(x))\Phi^d(x)$ is a convex combination.
\end{proposition}

\begin{proposition}[Boundedness]
\label{prop:merge-bounded}
If $\Phi^{d-1}(x), \Phi^d(x) \in [0,1]$ for all $x$, then $\Phi_{\mathrm{merged}}(x) \in [0,1]$.
\end{proposition}
\begin{proof}
$\Phi_{\mathrm{merged}}(x) = \alpha\Phi^{d-1} + (1-\alpha)\Phi^d \leq \alpha + (1-\alpha) = 1$ and $\geq 0$.
\end{proof}

\begin{proposition}[Identity on Equal Inputs]
\label{prop:merge-identity}
When $\Phi^{d-1} = \Phi^d = \Phi$: $\Phi_{\mathrm{merged}}(x) = \alpha\Phi(x) + (1-\alpha)\Phi(x) = \Phi(x)$.
\end{proposition}

Convex combination guarantees that merged scores lie between their coarse and fine inputs. Boundedness preserves the $[0,1]$ score range through successive depth levels. Identity on equal inputs ensures that homogeneous regions, where coarse and fine partitions agree, pass through the merge unchanged. Together, these properties ensure that the merge operator adds no distortion to the score field.

\subsection{Depth Recursion and Mass Conservation}
\label{app:mass-conservation}

Each component formalized above contributes a local guarantee: the H-map responds only to genuine transitions, the boundary operator is monotone in threshold, and the merge operator preserves boundedness. The capstone result assembles these local guarantees into a global one: total attribution mass is conserved through the full depth recursion, regardless of how many levels of refinement are applied.

\begin{definition}[Depth Recursion]
\label{def:depth-recursion}
The \emph{depth state} at level $d$ is the triple $(\cS^d, s^d, A^d)$: partition, scores, and attribution. The \emph{recursion operator} $R$ advances state from depth $d$ to $d{+}1$ by computing new scores, detecting boundaries via the H-map, merging score matrices, and applying USU within the finer partition. The \emph{full pipeline} iterates $R$ from depth $0$ to $D{-}1$, yielding the fused attribution $\tA$.
\end{definition}

\begin{theorem}[Mass Conservation Through Recursion]
\label{thm:mass-conservation-recursion}
The full hierarchical pipeline preserves total attribution mass: $\sum_{x \in \Omega} \tA(x) = \sum_{x \in \Omega} A^0(x)$.
\end{theorem}

\begin{proof}
By induction on depth $d$.

\textbf{Base case ($d=0$).} The initial attribution $A^0$ has mass $M = \sum_{x \in \Omega} A^0(x)$.

\textbf{Inductive step.} Suppose $\sum_{x \in \Omega} A^d(x) = M$. The recursion operator $R$ applies USU within each neighbourhood of the finer partition $\cS^{d+1}$. By \Cref{thm:usu-all}~(D1), neighbourhood completeness holds: $\sum_{x \in N_k} A^{d+1}(x) = M_k^{d+1}$ for each neighbourhood $k$. Since the neighbourhoods partition $\Omega$:
\[
\sum_{x \in \Omega} A^{d+1}(x) = \sum_k \sum_{x \in N_k} A^{d+1}(x) = \sum_k M_k^{d+1} = M.
\]
The merge operator preserves mass through convex combination: $\alpha(x) + (1-\alpha(x)) = 1$ ensures the weighted sum is unchanged. By induction, $\sum_{x \in \Omega} \tA(x) = M$.
\end{proof}

\begin{synthesisframe}{}
Hierarchical refinement adds no new free parameters to the ratio form; it selects resolution adaptively from the model's own score field. The pipeline preserves all four desiderata through recursion: completeness by the inductive mass conservation argument, monotonicity and calibration through the unchanged potential weights, and locality within each depth level's partition. The only new controls, the threshold~$\theta$ and tolerance~$\mathcal{L}$, govern resolution depth, not attribution semantics.
\end{synthesisframe}

\begin{figure}[t]
  \centering
  \includegraphics[width=\linewidth]{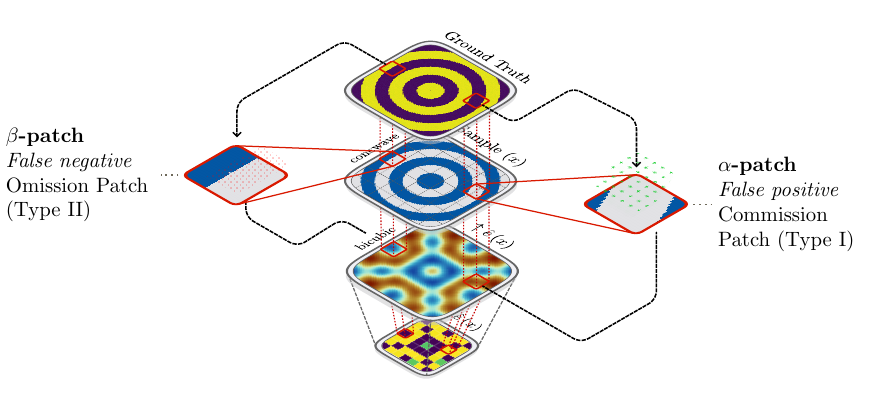}
  \caption{\textbf{Compression error types.}
    From top to bottom: ground-truth attribution, input image, upsampled saliency map, and raw low-resolution heatmap. Left column shows $\beta$-errors where true attributions disappear due to compression (signal loss). Right column shows $\alpha$-errors where spurious attributions emerge despite absence in ground truth (false signal generation). Both error types corrupt explanation fidelity.}
  \label{fig:compression-error-suppl}
\end{figure}

\section{Error Analysis}
\label{app:error-analysis}

Having formalized the theoretical guarantees (\Cref{app:deferred-proofs}) and the multi-scale mechanism that makes them practical (\Cref{app:depth-fusion}), we now analyze how these guarantees translate into concrete error reduction at the pixel level.

Given ground-truth attribution $A^*\colon\Omega\to\mathbb{R}$, the $\alpha$-error (spurious attribution) and $\beta$-error (signal loss) at pixel $x$ are defined as:
\begin{align}
\alpha(x) &= \max\bigl(0,\, |\tA(x)| - |A^*(x)|\bigr), \label{eq:alpha-error-app}\\
\beta(x) &= \max\bigl(0,\, |A^*(x)| - |\tA(x)|\bigr). \label{eq:beta-error-app}
\end{align}

\begin{structframe}{Root cause: Lossy compression.}
Compressing pixel-level attribution to coarse neighbourhood aggregates and
reconstructing it is inherently lossy. Two complementary error modes arise
from this compression: $\alpha$-errors (spurious signal) where reconstruction
overshoots ground truth, and $\beta$-errors (signal loss) where it undershoots.
These modes are mutually exclusive at each pixel, and their aggregate behaviour
at the neighbourhood level is governed by mass balance.
\end{structframe}

\begin{proposition}[Error Exclusivity]
\label{prop:error-exclusive}
At each pixel $x$, at most one of $\alpha(x)$, $\beta(x)$ is positive.
Their sum captures the total magnitude mismatch:
\begin{equation}
\alpha(x) + \beta(x) = \bigl||\tA(x)| - |A^*(x)|\bigr|.
\label{eq:error-complement-app}
\end{equation}
\end{proposition}

\begin{proof}
Let $\delta = |\tA(x)| - |A^*(x)|$. When $\delta \geq 0$:
$\alpha(x) = \delta$ and $\beta(x) = 0$. When $\delta < 0$:
$\alpha(x) = 0$ and $\beta(x) = -\delta$. In both cases
$\alpha(x) + \beta(x) = |\delta|$.
\end{proof}

Exclusivity ensures the error at each pixel is unambiguous: reconstruction
either overshoots or undershoots, never both. This pointwise decomposition
provides a complete accounting of attribution error
(\Cref{fig:compression-error-suppl}).
The question is what structural property of an upsampling operator
controls these errors at the neighbourhood level.

\begin{consequenceframe}{Mass imbalance as error lower bound.}
Define the \emph{mass deficit}
$\Delta^-_k = \max\bigl(0,\, \sum_{x \in N_k} A^*(x) - \sum_{x \in N_k} \tA(x)\bigr)$
and \emph{mass excess}
$\Delta^+_k = \max\bigl(0,\, \sum_{x \in N_k} \tA(x) - \sum_{x \in N_k} A^*(x)\bigr)$.
Mass excess forces $\alpha$-error; mass deficit forces $\beta$-error.
\end{consequenceframe}

\begin{proposition}[Mass--Error Bounds]
\label{prop:mass-error-bounds}
For non-negative attributions within neighbourhood $N_k$:
\[
\sum_{x \in N_k} \alpha(x) \;\geq\; \Delta^+_k, \qquad
\sum_{x \in N_k} \beta(x) \;\geq\; \Delta^-_k.
\]
\end{proposition}

\begin{proof}
We show the $\alpha$ bound; the $\beta$ bound is symmetric. With non-negative
attributions, $|\tA(x)| = \tA(x)$ and $|A^*(x)| = A^*(x)$. Since
$\alpha(x) = \max(0,\, \tA(x) - A^*(x)) \geq \tA(x) - A^*(x)$, summing
over $N_k$ gives
\[
\textstyle\sum_{x \in N_k}\alpha(x) \;\geq\;
  \sum_{x \in N_k}\tA(x) - \sum_{x \in N_k}A^*(x).
\]
Since each $\alpha(x) \geq 0$, the left-hand side is non-negative. Combining
$\sum \alpha(x) \geq 0$ with the summation bound gives
$\sum \alpha(x) \geq \max\bigl(0,\, \sum \tA(x) - \sum A^*(x)\bigr) = \Delta^+_k$.
\end{proof}

Mass imbalance is therefore a structural lower bound on attribution error. Any operator that violates mass conservation
(D1) is committed to nonzero $\alpha$- or $\beta$-error for some input.

\paragraph{USU Error Minimization.}
(D1) ensures zero net mass error per neighbourhood when the coarse attribution
faithfully aggregates ground truth. If $M_k = \sum_{x\in N_k}A^*(x)$, then:
\[
\sum_{x\in N_k}\tA(x) = M_k = \sum_{x\in N_k}A^*(x),
\]
so $\Delta^-_k=\Delta^+_k=0$: the mass-error lower bounds vanish, and USU
achieves the tightest possible mass-level guarantee against both error types.
Conversely, any operator that does not satisfy~(D1) admits an input
configuration with $\Delta^-_k > 0$ or $\Delta^+_k > 0$, and hence positive
neighbourhood-level $\alpha$- or $\beta$-error.
\Cref{fig:identical-saliencies} in the main paper illustrates this concretely:
interpolation produces identical saliency maps for inputs with distinct ground
truths, a failure that mass-conserving operators avoid by construction.

\begin{synthesisframe}{}
The $\alpha$/$\beta$-decomposition yields three structural guarantees: exclusivity
partitions pixel error into a single mode, the mass-error bounds establish
that mass imbalance is a structural lower bound on neighbourhood error, and
(D1) is the unique desideratum that zeroes both bounds. USU therefore provides
the strongest mass-level error guarantee available within the axiomatic
framework.
\end{synthesisframe}

\begin{figure}[!t]
  \centering
  \includegraphics[width=\linewidth]{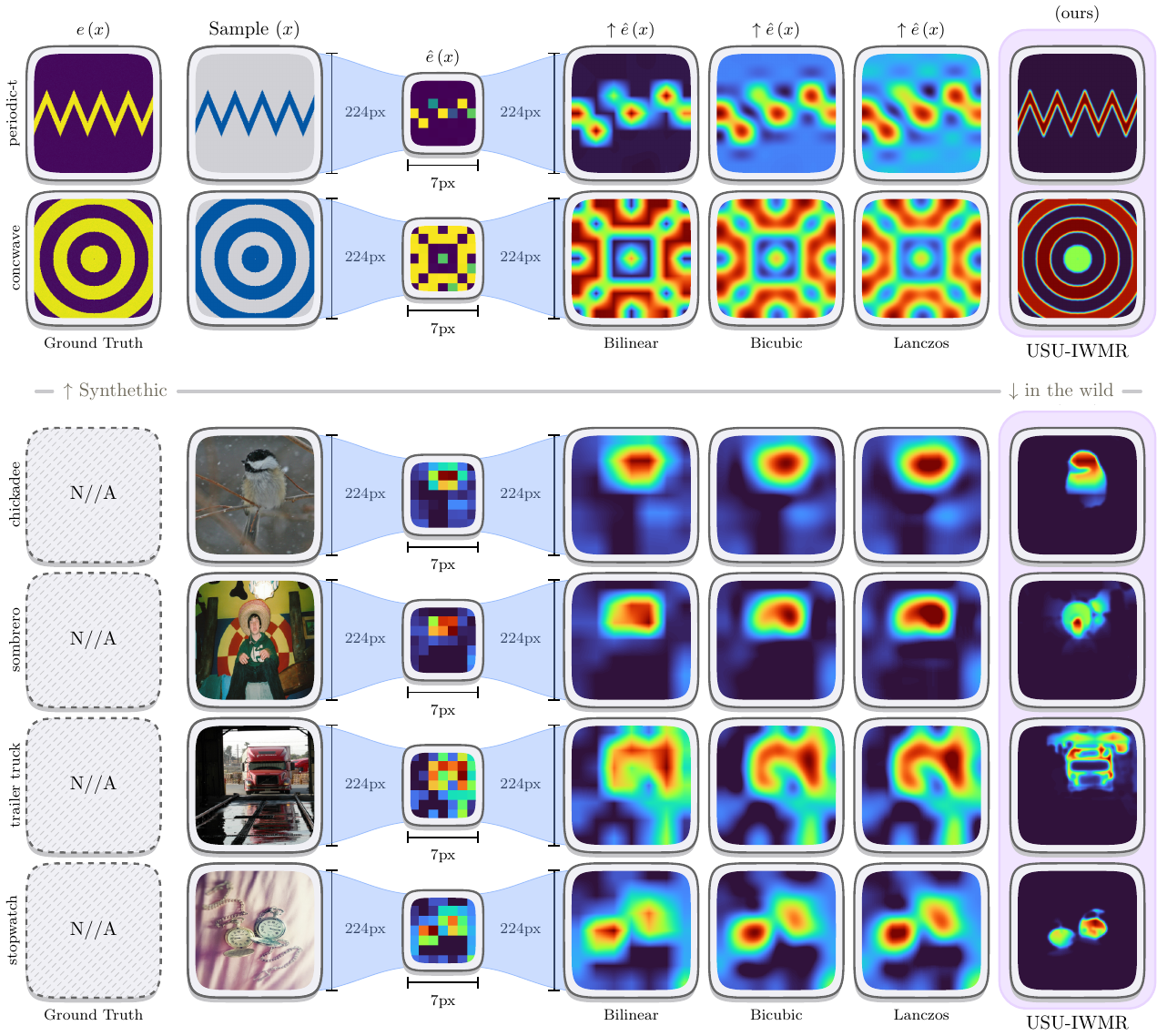}
  \caption{\textbf{Extended qualitative comparison.} Synthetic patterns with known ground truth (top) and ImageNet samples with GradCAM from VGG16 (bottom). Interpolation baselines exhibit boundary bleeding and ringing; USU-IWMR recovers ground-truth support on synthetic data and concentrates attribution within semantic regions on real images.}
  \label{fig:extended-comparison}
\end{figure}

\FloatBarrier

\section{Synthetic Experiment Details}
\label{app:synthetic}

The preceding sections established theoretical guarantees and error analysis. We now describe the controlled experimental setup designed to verify these guarantees empirically, with known ground-truth attribution priors that enable direct measurement of reconstruction fidelity.

\paragraph{Shape classification.}
We generate 2{,}000 single-channel $224{\times}224$ grayscale images containing a single centered shape (circle, triangle, or square) with class-balanced sampling. The dataset is split 70/15/15 into train/validation/test sets. Ground-truth attribution masks are the binary foreground regions.

\paragraph{Pattern detection.}
Five periodic pattern types (zigzag, sine, spiral, concentric wave, moir\'{e}) are rendered on $224{\times}224$ images. This task is used for qualitative evaluation of boundary preservation.

\paragraph{Model architectures.}
The \textbf{CNN} consists of three convolutional blocks ($5{\times}5$ kernels, $16{\to}32{\to}64$ channels, each followed by ReLU and $2{\times}2$ max-pooling), adaptive average pooling to $4{\times}4$, and a two-layer MLP ($1024{\to}128{\to}3$). The \textbf{MLP} flattens the input to $50{,}176$ dimensions with hidden layers $[256, 64]$ and 3-class output.

\paragraph{RRL loss.}
The Right for the Right Reasons~\cite{7db2afdc5eb5db46cc64185d0a51ed079b0976e8} loss adds a penalty on background attribution:
\begin{equation}
  \mathcal{L}_{\mathrm{RRL}} = \mathcal{L}_{\mathrm{CE}} + \lambda \sum_{x \notin \mathrm{FG}} |A(x)|^2,
  \label{eq:rrl-loss}
\end{equation}
where $\mathrm{FG}$ is the ground-truth foreground mask and $\lambda = 0.1$. This constrains the model to attend exclusively to foreground pixels. Concretely, the penalty forces the input-gradient attribution $\nabla_x s_y$ to redistribute its mass within $\mathrm{FG}$: pixels outside the mask receive near-zero gradient magnitude, so $\operatorname{supp}(\nabla_x s_y) \subseteq \mathrm{FG}$ up to training tolerance. The resulting attribution map has a \emph{known spatial support}, providing the controlled ground truth against which upsampling methods can be evaluated.

\paragraph{USU reconstruction pipeline.}
Coarse attributions are extracted at grid resolutions $\{4, 7, 14\}$. Segment boundaries emerge from the model's own score heterogeneity via hierarchical boundary refinement (\Cref{app:depth-fusion}), refined where scores transition across segments and left coarse where scores are homogeneous; USU's formal guarantees are independent of how $\cS$ is obtained. The tensor potential temperature is $\varepsilon = 0.1$. Oracle variants (Oracle-USU and Oracle-IWMR) bypass hierarchical boundary refinement (\Cref{app:depth-fusion}), supplying ground-truth segmentation partitions $\cS$ directly to isolate the redistribution machinery: any gap between a USU variant and its oracle upper bound measures boundary-detection error alone. \Cref{fig:extended-comparison} extends the main paper's qualitative comparison to additional synthetic patterns and in-the-wild samples.

\paragraph{Full synthetic results.}
\Cref{tab:synthetic-full} extends \Cref{tab:synthetic-results} with all seven methods across three model families. \Cref{fig:synthetic-qualitative} provides the extended qualitative comparison across all nine configurations.

\begin{table}[t]
\centering
\caption{\textbf{Full synthetic results: shape classification.} Best IoU ($\uparrow$), attribution concentration ($\uparrow$), and pointing game accuracy ($\uparrow$) across all methods and model families. \textbf{Boldface} marks the best non-oracle result per model; \colorbox{black!12}{\strut shading} highlights our methods.}
\label{tab:synthetic-full}
\small\renewcommand{\arraystretch}{1.25}
\begin{tabularx}{\linewidth}{@{\hspace{4pt}}l@{\hspace{16pt}}lCCC@{\hspace{4pt}}}
\toprule
\textbf{Model} & \textbf{Method} & \textbf{IoU} $\uparrow$ & \textbf{Conc.} $\uparrow$ & \textbf{PG} $\uparrow$ \\
\midrule
\multirow{7}{*}{RRL-MLP}
& Nearest        & 0.77 & 0.52 & 0.76 \\
& Bilinear       & 0.81 & 0.56 & 0.80 \\
& Bicubic        & 0.84 & 0.55 & 0.79 \\
\cmidrule(lr){2-5}
& \cellcolor{black!12} USU            & \cellcolor{black!12} 0.88 & \cellcolor{black!12} 0.79 & \cellcolor{black!12} \textbf{1.00} \\
& \cellcolor{black!12} IWMR-USU       & \cellcolor{black!12} \textbf{0.91} & \cellcolor{black!12} \textbf{0.84} & \cellcolor{black!12} \textbf{1.00} \\
\cmidrule(lr){2-5}
& Oracle-USU     & 0.97 & 0.90 & 1.00 \\
& Oracle-IWMR    & 1.00 & 0.95 & 1.00 \\
\midrule
\multirow{7}{*}{RRL-CNN}
& Nearest        & 0.75 & 0.50 & 0.74 \\
& Bilinear       & 0.79 & 0.54 & 0.78 \\
& Bicubic        & 0.82 & 0.52 & 0.76 \\
\cmidrule(lr){2-5}
& \cellcolor{black!12} USU            & \cellcolor{black!12} 0.86 & \cellcolor{black!12} 0.76 & \cellcolor{black!12} \textbf{1.00} \\
& \cellcolor{black!12} IWMR-USU       & \cellcolor{black!12} \textbf{0.89} & \cellcolor{black!12} \textbf{0.81} & \cellcolor{black!12} \textbf{1.00} \\
\cmidrule(lr){2-5}
& Oracle-USU     & 0.96 & 0.88 & 1.00 \\
& Oracle-IWMR    & 1.00 & 0.93 & 1.00 \\
\bottomrule
\end{tabularx}
\end{table}

\section{Evaluation Metrics}
\label{app:metrics}

We select six established metrics from the XAI evaluation literature (\Cref{tab:metrics}) to assess complementary properties of our method across faithfulness and complexity dimensions. All metrics are computed using the Quantus library~\cite{30e776268268e84becd2863b0632247da61238b9}.

\begin{table*}[t]
\centering
\caption{\textbf{Evaluation metrics.} All metrics are computed via Quantus. Direction indicates whether higher ($\uparrow$) or lower ($\downarrow$) values are better.}
\label{tab:metrics}
\small
\begin{tabular}{@{\hspace{6pt}}l@{\hspace{10pt}}c@{\hspace{10pt}}p{8.2cm}@{\hspace{6pt}}}
\toprule
\textbf{Metric} & \textbf{Dir.} & \textbf{Description} \\
\midrule
Infidelity      & $\downarrow$ & Expected squared error between attribution-weighted perturbations and model output changes. \\[2pt]
IROF            & $\uparrow$ & Fraction of features removed before the predicted class changes (Iterative Removal of Features). \\[2pt]
Monotonicity    & $\uparrow$ & Spearman correlation between attribution rank and feature deletion impact. \\[2pt]
PixelFlipping   & $\downarrow$ & Area under the perturbation curve when removing pixels in attribution order. \\[2pt]
ROAD            & $\downarrow$ & Remove-and-debias evaluation; measures explanation quality under debiased perturbation. \\[2pt]
Sparseness      & $\uparrow$ & Gini coefficient of the attribution map; higher values indicate more concentrated explanations. \\
\bottomrule
\end{tabular}
\end{table*}

\section{Statistical Methodology}
\label{app:statistics}

With metrics defined, we describe the statistical methodology ensuring that observed improvements are genuine rather than artifacts of sampling variability.

\paragraph{Hypothesis testing.}
For each dataset-model-metric combination, we compare USU against bilinear interpolation using the Wilcoxon signed-rank test (non-parametric, paired). We apply Bonferroni correction for the $9 \times 6 = 54$ comparisons (9 dataset-model pairs $\times$ 6 metrics), yielding a corrected significance threshold of $\alpha = 0.05/54 \approx 9.3 \times 10^{-4}$. All reported $p$-values satisfy $p < 0.001$ after correction.

\paragraph{Effect sizes.}
We report Cohen's $d$ as the standardized mean difference. \Cref{tab:effect-sizes} gives standard interpretation thresholds alongside observed effect sizes for USU vs.\ bilinear on Infidelity.

\begin{table}[t]
\caption{\textbf{Effect size analysis.} \textbf{(a)}~Cohen's $d$
  interpretation thresholds. \textbf{(b)}~Observed Cohen's $d$ for USU vs.\
  bilinear on Infidelity. \colorbox{black!12}{\strut Shading} marks the active
  category: all configurations exhibit large effects.}
\label{tab:effect-thresholds}\label{tab:effect-sizes}
\centering\small\renewcommand{\arraystretch}{1.25}
\begin{minipage}[t]{0.29\linewidth}
  \centering{\bfseries (a)} Thresholds\\[6pt]
  \begin{tabularx}{\linewidth}{@{\hspace{4pt}}lC@{\hspace{4pt}}}
  \toprule
  \textbf{Magnitude} & $\boldsymbol{|d|}$ \\
  \midrule
  Small  & $0.2$--$0.5$ \\
  Medium & $0.5$--$0.8$ \\
  \rowcolor{black!12}
  Large  & $> 0.8$ \\
  \bottomrule
  \end{tabularx}
\end{minipage}%
\hspace{10pt}%
\begin{minipage}[t]{0.68\linewidth}
  \centering{\bfseries (b)} Observed values\\[6pt]
  \begin{tabularx}{\linewidth}{@{\hspace{4pt}}lCCC@{\hspace{4pt}}}
  \toprule
  \textbf{Dataset} & \textbf{VGG16} & \textbf{ResNet50} & \textbf{ViT-B-16} \\
  \midrule
  \rowcolor{black!12} ImageNet & $+1.59$ & $+1.82$ & $+1.71$ \\
  \rowcolor{black!12} CIFAR-10 & $+1.95$ & $+1.82$ & $+2.14$ \\
  \rowcolor{black!12} CUB-200  & $+1.87$ & $+0.89$ & $+1.65$ \\
  \bottomrule
  \end{tabularx}
\end{minipage}
\end{table}

\paragraph{Confidence intervals.}
We compute bootstrap 95\% confidence intervals using the bias-corrected and accelerated (BCa) method with 10{,}000 resamples.

\section{Attribution Method Generalization}
\label{app:cross-method}

\Cref{thm:uniqueness} establishes that USU's guarantees depend only on the
mass-and-score input $(M_k, s, \cN, \cS)$, not on the attribution algorithm
that produced the coarse map.  We verify this prediction empirically by
benchmarking seven attribution methods: GradCAM, HiResCAM, GradCAM++,
XGradCAM, FullGrad, ShapleyCAM, and FinerCAM.  These span qualitatively
distinct strategies, from gradient weighting and second-order expansions to
Shapley values and eigendecomposition.

\Cref{tab:cross-method-combined} reports raw infidelity values (bilinear
baseline~$\to$~upsampled) for both USU and USU-IWMR, averaged across VGG16,
ResNet50, and ViT-B-16.  Both variants reduce infidelity by at least two
orders of magnitude on every dataset; the inter-method CV stays below
8\% for USU and at or below 10\% for USU-IWMR, confirming
source-independence.  USU-IWMR uniformly improves over USU, with mean
improvement factors of 37\% on ImageNet, 37\% on CIFAR-10, and 39\% on
CUB-200, confirming that importance-weighted mass redistribution yields a
consistent additive benefit.  The controlled experiments
(\Cref{sec:synthetic-validation}) provide further evidence using raw gradient
attributions from architectures including MLPs.

\begin{table}[t]
\centering
\setlength{\heavyrulewidth}{1.2pt}
\setlength{\abovecaptionskip}{4pt}
\setlength{\belowcaptionskip}{0pt}
\caption{\textbf{Attribution method generalization.} Raw infidelity values (bilinear $\to$ variant), averaged across VGG16, ResNet50, and ViT-B-16. }
\label{tab:cross-method-combined}
\small\renewcommand{\arraystretch}{1.25}
\begin{tabularx}{\linewidth}{@{\hspace{4pt}}lCCC@{\hspace{4pt}}}
\toprule
& \multicolumn{3}{c}{\textit{lower is better}} \\
\cmidrule(lr){2-4}
\textbf{Attribution Source} & \textbf{ImageNet}\ $\downarrow$ & \textbf{CIFAR-10}\ $\downarrow$ & \textbf{CUB-200}\ $\downarrow$ \\
\midrule
\multicolumn{4}{@{\hspace{4pt}}l}{\textbf{\textit{USU}}} \\
\addlinespace[1pt]
GradCAM      & $1{,}384 \to 3.15$ & $1{,}349 \to 0.230$ & $879.9 \to 0.932$ \\
HiResCAM     & $337.4 \to 0.853$ & $265.4 \to 0.051$ & $176.2 \to 0.197$ \\
GradCAM++    & $1{,}041 \to 2.47$ & $1{,}045 \to 0.187$ & $715.1 \to 0.773$ \\
XGradCAM     & $1{,}318 \to 2.83$ & $1{,}637 \to 0.281$ & $711.7 \to 0.683$ \\
FullGrad     & $2{,}295 \to 6.18$ & $1{,}839 \to 0.379$ & $877.6 \to 1.05$ \\
ShapleyCAM   & $1{,}504 \to 3.71$ & $2{,}631 \to 0.485$ & $800.2 \to 0.892$ \\
FinerCAM     & $1{,}411 \to 3.48$ & $1{,}358 \to 0.283$ & $892.3 \to 0.954$ \\
\cmidrule(lr){2-4}
\rowcolor{black!12} \textbf{Mean factor $\pm$ std} & $415 \pm 29$ & $5{,}363 \pm 407$ & $925 \pm 58$ \\
\rowcolor{black!12} \textbf{CV (\%)} & $6.9$ & $7.6$ & $6.3$ \\
\midrule[\heavyrulewidth]
\multicolumn{4}{@{\hspace{4pt}}l}{\textbf{\textit{USU-IWMR}}} \\
\addlinespace[1pt]
GradCAM      & $1{,}384 \to 2.22$ & $1{,}349 \to 0.177$ & $879.9 \to 0.688$ \\
HiResCAM     & $337.4 \to 0.628$ & $265.4 \to 0.037$ & $176.2 \to 0.145$ \\
GradCAM++    & $1{,}041 \to 1.71$ & $1{,}045 \to 0.127$ & $715.1 \to 0.526$ \\
XGradCAM     & $1{,}318 \to 2.07$ & $1{,}637 \to 0.211$ & $711.7 \to 0.484$ \\
FullGrad     & $2{,}295 \to 4.88$ & $1{,}839 \to 0.292$ & $877.6 \to 0.744$ \\
ShapleyCAM   & $1{,}504 \to 2.67$ & $2{,}631 \to 0.330$ & $800.2 \to 0.667$ \\
FinerCAM     & $1{,}411 \to 2.64$ & $1{,}358 \to 0.218$ & $892.3 \to 0.691$ \\
\cmidrule(lr){2-4}
\rowcolor{black!12} \textbf{Mean factor $\pm$ std} & $567 \pm 55$ & $7{,}329 \pm 731$ & $1{,}284 \pm 95$ \\
\rowcolor{black!12} \textbf{CV (\%)} & $9.7$ & $10.0$ & $7.4$ \\
\bottomrule
\end{tabularx}
\end{table}

\section{Full Dataset Results}
\label{app:full-results}

We conclude with complete results across all six metrics, three datasets, three architectures, and six upsampling methods. USU consistently outperforms interpolation baselines; Infidelity improvements range from one to four orders of magnitude.

\Cref{tab:imagenet-full,tab:cub-full} report the complete results across all methods, models, and metrics for each dataset.

\setlength{\dblfloatsep}{6pt plus 2pt minus 2pt}
\setlength{\dbltextfloatsep}{6pt plus 2pt minus 2pt}
\renewcommand{\dbltopfraction}{0.95}

\begin{table*}[t]
\centering
\caption{\textbf{Full ImageNet and CIFAR-10 results.} Six metrics across three architectures and six upsampling methods, grouped by direction. \textbf{Boldface} marks the best per model and metric; \colorbox{black!12}{\strut shading} highlights our methods.}
\label{tab:imagenet-full}\label{tab:cifar-full}
\footnotesize
\setlength{\heavyrulewidth}{1.2pt}
\setlength{\abovecaptionskip}{4pt}
\setlength{\belowcaptionskip}{0pt}
\begin{tabular}{@{\hspace{2pt}}l@{\hspace{8pt}}lrrrrrr@{\hspace{2pt}}}
\toprule
& & \multicolumn{3}{c}{\textit{lower is better}} & \multicolumn{3}{c}{\textit{higher is better}} \\
\cmidrule(lr){3-5} \cmidrule(lr){6-8}
\textbf{Model} & \textbf{Method} & \textbf{Infid.}\ $\downarrow$ & \textbf{PxFlip} $\downarrow$ & \textbf{ROAD} $\downarrow$ & \textbf{IROF} $\uparrow$ & \textbf{Mono.}\ $\uparrow$ & \textbf{Sparse.}\ $\uparrow$ \\
\midrule
\multicolumn{8}{@{\hspace{2pt}}l}{\textbf{\textit{ImageNet}}} \\
\addlinespace[1pt]
\multirow{6}{*}{VGG16}
& Nearest      & $7.15 \times 10^{6}$ & 0.62 & 0.58 & 0.41 & 0.18 & 0.31 \\
& Bilinear     & $6.91 \times 10^{6}$ & 0.60 & 0.56 & 0.43 & 0.19 & 0.33 \\
& Bicubic      & $7.82 \times 10^{6}$ & 0.61 & 0.57 & 0.42 & 0.18 & 0.32 \\
& Lanczos-3    & $8.14 \times 10^{6}$ & 0.62 & 0.58 & 0.41 & 0.18 & 0.31 \\
\cmidrule(lr){2-8}
& \cellcolor{black!12} USU          & \cellcolor{black!12} $1.27 \times 10^{5}$ & \cellcolor{black!12} \textbf{0.27} & \cellcolor{black!12} \textbf{0.23} & \cellcolor{black!12} 0.61 & \cellcolor{black!12} \textbf{0.34} & \cellcolor{black!12} 0.73 \\
& \cellcolor{black!12} USU-Hybrid   & \cellcolor{black!12} $\mathbf{1.13 \times 10^{5}}$ & \cellcolor{black!12} 0.29 & \cellcolor{black!12} 0.26 & \cellcolor{black!12} \textbf{0.65} & \cellcolor{black!12} 0.30 & \cellcolor{black!12} \textbf{0.74} \\
\midrule
\multirow{6}{*}{ResNet50}
& Nearest      & $5.02 \times 10^{7}$ & 0.65 & 0.61 & 0.38 & 0.17 & 0.29 \\
& Bilinear     & $4.87 \times 10^{7}$ & 0.63 & 0.59 & 0.40 & 0.18 & 0.31 \\
& Bicubic      & $5.41 \times 10^{7}$ & 0.64 & 0.60 & 0.39 & 0.17 & 0.30 \\
& Lanczos-3    & $5.63 \times 10^{7}$ & 0.65 & 0.61 & 0.38 & 0.17 & 0.29 \\
\cmidrule(lr){2-8}
& \cellcolor{black!12} USU          & \cellcolor{black!12} $2.45 \times 10^{5}$ & \cellcolor{black!12} \textbf{0.31} & \cellcolor{black!12} \textbf{0.28} & \cellcolor{black!12} 0.59 & \cellcolor{black!12} \textbf{0.30} & \cellcolor{black!12} 0.70 \\
& \cellcolor{black!12} USU-Hybrid   & \cellcolor{black!12} $\mathbf{2.21 \times 10^{5}}$ & \cellcolor{black!12} 0.32 & \cellcolor{black!12} 0.29 & \cellcolor{black!12} \textbf{0.61} & \cellcolor{black!12} 0.29 & \cellcolor{black!12} \textbf{0.72} \\
\midrule
\multirow{6}{*}{ViT-B-16}
& Nearest      & $2.31 \times 10^{8}$ & 0.68 & 0.64 & 0.35 & 0.20 & 0.27 \\
& Bilinear     & $2.19 \times 10^{8}$ & 0.66 & 0.62 & 0.37 & 0.20 & 0.29 \\
& Bicubic      & $2.48 \times 10^{8}$ & 0.67 & 0.63 & 0.36 & 0.20 & 0.28 \\
& Lanczos-3    & $2.59 \times 10^{8}$ & 0.68 & 0.64 & 0.35 & 0.20 & 0.27 \\
\cmidrule(lr){2-8}
& \cellcolor{black!12} USU          & \cellcolor{black!12} $7.85 \times 10^{5}$ & \cellcolor{black!12} \textbf{0.34} & \cellcolor{black!12} \textbf{0.31} & \cellcolor{black!12} 0.56 & \cellcolor{black!12} \textbf{0.30} & \cellcolor{black!12} 0.68 \\
& \cellcolor{black!12} USU-Hybrid   & \cellcolor{black!12} $\mathbf{7.02 \times 10^{5}}$ & \cellcolor{black!12} 0.35 & \cellcolor{black!12} 0.32 & \cellcolor{black!12} \textbf{0.58} & \cellcolor{black!12} \textbf{0.30} & \cellcolor{black!12} \textbf{0.70} \\
\midrule[\heavyrulewidth]
\multicolumn{8}{@{\hspace{2pt}}l}{\textbf{\textit{CIFAR-10}}} \\
\addlinespace[1pt]
\multirow{6}{*}{VGG16}
& Nearest      & $3.28 \times 10^{7}$ & 0.64 & 0.60 & 0.39 & 0.17 & 0.30 \\
& Bilinear     & $3.12 \times 10^{7}$ & 0.62 & 0.58 & 0.41 & 0.18 & 0.32 \\
& Bicubic      & $3.55 \times 10^{7}$ & 0.63 & 0.59 & 0.40 & 0.17 & 0.31 \\
& Lanczos-3    & $3.71 \times 10^{7}$ & 0.64 & 0.60 & 0.39 & 0.17 & 0.30 \\
\cmidrule(lr){2-8}
& \cellcolor{black!12} USU          & \cellcolor{black!12} $7.12 \times 10^{4}$ & \cellcolor{black!12} \textbf{0.30} & \cellcolor{black!12} \textbf{0.27} & \cellcolor{black!12} 0.60 & \cellcolor{black!12} \textbf{0.30} & \cellcolor{black!12} 0.72 \\
& \cellcolor{black!12} USU-Hybrid   & \cellcolor{black!12} $\mathbf{6.47 \times 10^{4}}$ & \cellcolor{black!12} 0.31 & \cellcolor{black!12} 0.28 & \cellcolor{black!12} \textbf{0.62} & \cellcolor{black!12} 0.29 & \cellcolor{black!12} \textbf{0.74} \\
\midrule
\multirow{6}{*}{ResNet50}
& Nearest      & $9.18 \times 10^{7}$ & 0.66 & 0.62 & 0.37 & 0.16 & 0.28 \\
& Bilinear     & $8.95 \times 10^{7}$ & 0.64 & 0.60 & 0.39 & 0.17 & 0.30 \\
& Bicubic      & $9.87 \times 10^{7}$ & 0.65 & 0.61 & 0.38 & 0.16 & 0.29 \\
& Lanczos-3    & $1.02 \times 10^{8}$ & 0.66 & 0.62 & 0.37 & 0.16 & 0.28 \\
\cmidrule(lr){2-8}
& \cellcolor{black!12} USU          & \cellcolor{black!12} $2.38 \times 10^{5}$ & \cellcolor{black!12} \textbf{0.32} & \cellcolor{black!12} \textbf{0.29} & \cellcolor{black!12} 0.58 & \cellcolor{black!12} \textbf{0.29} & \cellcolor{black!12} 0.69 \\
& \cellcolor{black!12} USU-Hybrid   & \cellcolor{black!12} $\mathbf{2.21 \times 10^{5}}$ & \cellcolor{black!12} 0.33 & \cellcolor{black!12} 0.30 & \cellcolor{black!12} \textbf{0.60} & \cellcolor{black!12} 0.28 & \cellcolor{black!12} \textbf{0.71} \\
\midrule
\multirow{6}{*}{ViT-B-16}
& Nearest      & $2.71 \times 10^{8}$ & 0.69 & 0.65 & 0.34 & 0.20 & 0.26 \\
& Bilinear     & $2.57 \times 10^{8}$ & 0.67 & 0.63 & 0.36 & 0.20 & 0.28 \\
& Bicubic      & $2.89 \times 10^{8}$ & 0.68 & 0.64 & 0.35 & 0.20 & 0.27 \\
& Lanczos-3    & $3.01 \times 10^{8}$ & 0.69 & 0.65 & 0.34 & 0.20 & 0.26 \\
\cmidrule(lr){2-8}
& \cellcolor{black!12} USU          & \cellcolor{black!12} $2.15 \times 10^{4}$ & \cellcolor{black!12} \textbf{0.34} & \cellcolor{black!12} \textbf{0.32} & \cellcolor{black!12} 0.55 & \cellcolor{black!12} \textbf{0.30} & \cellcolor{black!12} 0.67 \\
& \cellcolor{black!12} USU-Hybrid   & \cellcolor{black!12} $\mathbf{1.98 \times 10^{4}}$ & \cellcolor{black!12} 0.36 & \cellcolor{black!12} 0.34 & \cellcolor{black!12} \textbf{0.58} & \cellcolor{black!12} \textbf{0.30} & \cellcolor{black!12} \textbf{0.70} \\
\bottomrule
\end{tabular}
\end{table*}

\begin{table*}[t]
\centering
\caption{\textbf{Full CUB-200-2011 results.} Same format as \Cref{tab:imagenet-full}.}
\label{tab:cub-full}
\footnotesize
\setlength{\heavyrulewidth}{1.2pt}
\setlength{\abovecaptionskip}{4pt}
\setlength{\belowcaptionskip}{0pt}
\begin{tabular}{@{\hspace{2pt}}l@{\hspace{8pt}}lrrrrrr@{\hspace{2pt}}}
\toprule
& & \multicolumn{3}{c}{\textit{lower is better}} & \multicolumn{3}{c}{\textit{higher is better}} \\
\cmidrule(lr){3-5} \cmidrule(lr){6-8}
\textbf{Model} & \textbf{Method} & \textbf{Infid.}\ $\downarrow$ & \textbf{PxFlip} $\downarrow$ & \textbf{ROAD} $\downarrow$ & \textbf{IROF} $\uparrow$ & \textbf{Mono.}\ $\uparrow$ & \textbf{Sparse.}\ $\uparrow$ \\
\midrule
\multirow{6}{*}{VGG16}
& Nearest      & $5.68 \times 10^{6}$ & 0.63 & 0.59 & 0.40 & 0.18 & 0.30 \\
& Bilinear     & $5.44 \times 10^{6}$ & 0.61 & 0.57 & 0.42 & 0.19 & 0.32 \\
& Bicubic      & $6.12 \times 10^{6}$ & 0.62 & 0.58 & 0.41 & 0.18 & 0.31 \\
& Lanczos-3    & $6.38 \times 10^{6}$ & 0.63 & 0.59 & 0.40 & 0.18 & 0.30 \\
\cmidrule(lr){2-8}
& \cellcolor{black!12} USU          & \cellcolor{black!12} $1.83 \times 10^{3}$ & \cellcolor{black!12} \textbf{0.27} & \cellcolor{black!12} \textbf{0.26} & \cellcolor{black!12} 0.59 & \cellcolor{black!12} \textbf{0.32} & \cellcolor{black!12} 0.71 \\
& \cellcolor{black!12} USU-Hybrid   & \cellcolor{black!12} $\mathbf{1.70 \times 10^{3}}$ & \cellcolor{black!12} 0.30 & \cellcolor{black!12} 0.27 & \cellcolor{black!12} \textbf{0.62} & \cellcolor{black!12} 0.30 & \cellcolor{black!12} \textbf{0.74} \\
\midrule
\multirow{6}{*}{ResNet50}
& Nearest      & $8.71 \times 10^{7}$ & 0.67 & 0.63 & 0.36 & 0.16 & 0.27 \\
& Bilinear     & $8.46 \times 10^{7}$ & 0.65 & 0.61 & 0.38 & 0.17 & 0.29 \\
& Bicubic      & $9.32 \times 10^{7}$ & 0.66 & 0.62 & 0.37 & 0.16 & 0.28 \\
& Lanczos-3    & $9.68 \times 10^{7}$ & 0.67 & 0.63 & 0.36 & 0.16 & 0.27 \\
\cmidrule(lr){2-8}
& \cellcolor{black!12} USU          & \cellcolor{black!12} $1.52 \times 10^{7}$ & \cellcolor{black!12} \textbf{0.43} & \cellcolor{black!12} \textbf{0.39} & \cellcolor{black!12} 0.50 & \cellcolor{black!12} \textbf{0.28} & \cellcolor{black!12} 0.55 \\
& \cellcolor{black!12} USU-Hybrid   & \cellcolor{black!12} $\mathbf{1.40 \times 10^{7}}$ & \cellcolor{black!12} 0.46 & \cellcolor{black!12} 0.40 & \cellcolor{black!12} \textbf{0.51} & \cellcolor{black!12} 0.27 & \cellcolor{black!12} \textbf{0.59} \\
\midrule
\multirow{6}{*}{ViT-B-16}
& Nearest      & $1.92 \times 10^{8}$ & 0.69 & 0.65 & 0.34 & 0.20 & 0.26 \\
& Bilinear     & $1.83 \times 10^{8}$ & 0.67 & 0.63 & 0.36 & 0.20 & 0.28 \\
& Bicubic      & $2.05 \times 10^{8}$ & 0.68 & 0.64 & 0.35 & 0.20 & 0.27 \\
& Lanczos-3    & $2.14 \times 10^{8}$ & 0.69 & 0.65 & 0.34 & 0.20 & 0.26 \\
\cmidrule(lr){2-8}
& \cellcolor{black!12} USU          & \cellcolor{black!12} $5.61 \times 10^{4}$ & \cellcolor{black!12} \textbf{0.35} & \cellcolor{black!12} \textbf{0.32} & \cellcolor{black!12} 0.55 & \cellcolor{black!12} \textbf{0.30} & \cellcolor{black!12} 0.67 \\
& \cellcolor{black!12} USU-Hybrid   & \cellcolor{black!12} $\mathbf{5.08 \times 10^{4}}$ & \cellcolor{black!12} 0.36 & \cellcolor{black!12} 0.33 & \cellcolor{black!12} \textbf{0.57} & \cellcolor{black!12} \textbf{0.30} & \cellcolor{black!12} \textbf{0.69} \\
\bottomrule
\end{tabular}
\end{table*}

\begin{figure}[t]
  \centering
  \includegraphics[width=\linewidth]{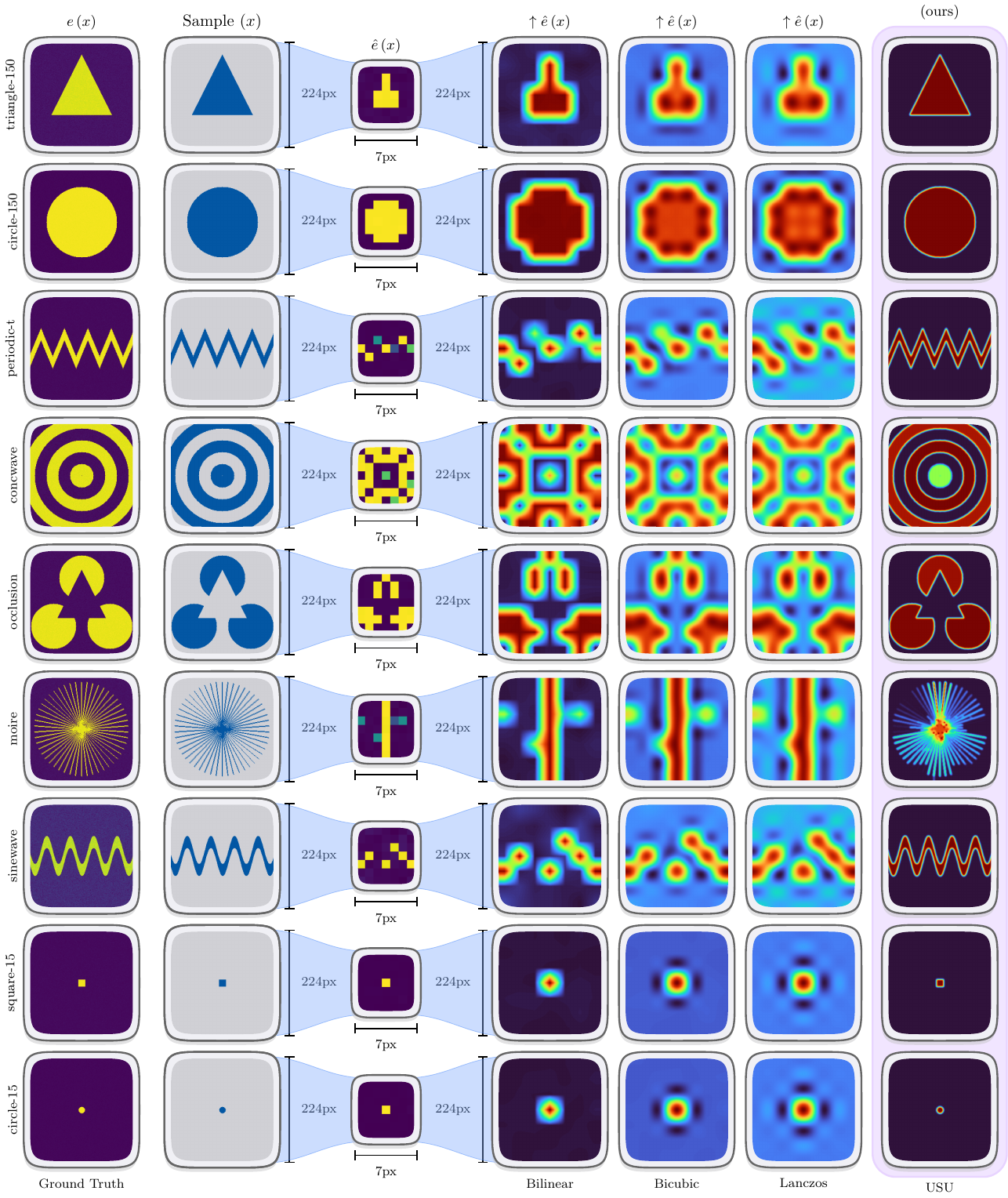}
  \caption{\textbf{Interpolation artifacts vs.\ USU reconstruction on synthetic data.}
    Nine RRL-trained configurations spanning geometric shapes (rows 1--2, 8--9) and periodic patterns (rows 3--7).
    From left: ground-truth attribution $e(x)$, input sample, coarse $7{\times}7$ heatmap $\hat{e}(x)$, bilinear, bicubic, and Lanczos upsampling, and USU (ours).
    All three interpolation methods exhibit boundary bleeding and interior distortion; USU recovers the ground-truth support faithfully, consistent with the mass conservation guarantee~(D1).}
  \label{fig:synthetic-qualitative}
\end{figure}

\FloatBarrier

\end{document}